\documentclass{article} %
\usepackage{iclr2027_conference,times}

\usepackage{amsmath,amsfonts,bm}

\def\eqref#1{equation~\ref{#1}}

\def\1{\bm{1}}

\DeclareMathAlphabet{\mathsfit}{\encodingdefault}{\sfdefault}{m}{sl}
\SetMathAlphabet{\mathsfit}{bold}{\encodingdefault}{\sfdefault}{bx}{n}

\def\gB{{\mathcal{B}}}

\def\gD{{\mathcal{D}}}

\def\gF{{\mathcal{F}}}
\def\gG{{\mathcal{G}}}

\def\gO{{\mathcal{O}}}

\def\sS{{\mathbb{S}}}

\newcommand{\E}{\mathbb{E}}

\newcommand{\R}{\mathbb{R}}

\usepackage{hyperref}
\usepackage{url}
\usepackage{bm}
\usepackage{bbm}
\usepackage{todonotes}
\usepackage{apxproof}
\usepackage{refcount}
\usepackage{tikz-cd}

\usepackage{amsmath, mathtools, amssymb, amsfonts, amsthm}
\theoremstyle{plain}
\newtheorem{lemma}{Lemma}
\newtheorem{theorem}{Theorem}
\newtheorem{proposition}{Proposition}
\theoremstyle{definition}
\newtheorem{assumption}{Assumption}
\newtheorem{definition}{Definition}
\newtheorem{remark}{Remark}
\newtheorem{example}{Example}
\newtheorem*{mainrestated}{Theorem~\ref{thm:main}}

\DeclareMathOperator{\clip}{\operatorname{clip}}

\DeclareMathOperator{\prev}{\operatorname{prev}}
\DeclareMathOperator{\Agg}{\operatorname{Agg}}

\DeclareMathOperator{\indi}{\mathbf{1}}
\newcommand{\nexxt}{\operatorname{next}}
\renewcommand{\prev}{\operatorname{prev}}

\newcommand{\myalgo}{{\textsc{Throttle}}}

\title{Robustifying Asynchronous SGD via Soft Throttling}

\author{Kaoru Otsuka\thanks{Equal contribution.} \\
Okinawa Institute of Science and Technology \\
\texttt{kaoru.otsuka@oist.jp}
\And
Maxime Meyer\footnotemark[1]\ \thanks{Work done during the internship at OIST.} \\
National University of Singapore \\
\texttt{maxime.meyer@u.nus.edu}
\And
Yuki Takezawa \\
Toyota Motor Corporation \\
\texttt{yuki\_takezawa@mail.toyota.co.jp} 
\And
Makoto Yamada \\
Okinawa Institute of Science and Technology \\
\texttt{makoto.yamada@oist.jp}
\And
Anastasia Koloskova \\
University of Zurich \\
\texttt{anastasiia.koloskova@uzh.ch}
}

\iclrfinalcopy %
\usepackage{fancyhdr}   %
\usepackage{xcolor}

\usepackage{mathtools}
\usepackage{float}
\usepackage{algorithm}
\usepackage{footnote}
\usepackage{booktabs} 
\makesavenoteenv{algorithm}
\usepackage[noend]{algpseudocode}
\usepackage{cancel}
\usepackage{tikz}
\usetikzlibrary{arrows.meta,patterns,calc}
\definecolor{yred}{HTML}{E83850}
\definecolor{zblue}{HTML}{2476D4}
\definecolor{resetgray}{HTML}{7B8491}


\begin{document}

\maketitle

\AtBeginDocument{\lhead{Preprint}} %

\begin{abstract}
    Asynchronous SGD is a popular algorithm for distributed learning where each client's gradient update is applied on arrival.
    This leads to a speed-up, but also an increased vulnerability to attacks, as fast clients can dominate the total update.
    We introduce {\myalgo}, a Byzantine-robust generalization of asynchronous SGD where the key idea is to exponentially down-weight updates from faster clients by a factor $q$.
    Both asynchronous SGD ($q=1$) and synchronous Byzantine-robust SGD ($q\to\infty$) correspond to specific settings of {\myalgo}.
    We provide a theoretical analysis of the convergence rate and validate the robustness to attacks both theoretically and empirically.
    Remarkably, our experiments show that this down-weighting mechanism can also improve performance over standard asynchronous SGD even in the non-Byzantine setting.
\end{abstract}

\section{Introduction}

Distributed learning trains a shared model faster by spreading the training workload across many clients. It has become essential for large-scale machine learning, and it has been widely studied in recent years as a result~\citep{yang2019survey, verbraeken2020survey}.
Typically, distributed learning consists of a succession of steps. At every step, i) the clients send their local gradients, ii) the central server performs its update once all of them are received, and iii) the updated model is sent to the clients so that the next computation step can start. One clear bottleneck is that this process can only be as fast as the \textit{slowest} machine. It leads to considerable slowdowns in practice as clients typically compute at very different speeds. This can arise from heterogeneous computational power~\citep{horvath2021fjord}, network latency, faulty devices~\citep{ryabinin2021moshpit}, or slow \textit{straggler} clients in GPU clusters~\citep{chen2016revisiting}. A natural remedy is to leverage \textit{asynchronous} SGD methods~\citep{recht2011hogwild, mania2017perturbed, nguyen2022federated}, where the central server updates the model at each feedback, and the corresponding client starts computing the next gradient immediately. These are the methods we study in this paper.

Another limitation of distributed learning is that, while integrating more clients increases the amount of data and compute, it also increases the risk of having malicious or faulty clients~\citep{xie_fall_2019, baruch2019nips, kairouz2021advances}. We refer to those as Byzantine clients~\citep{lamport_byzantine_2019}. This is especially problematic for asynchronous methods, where one malicious client can send an overwhelming number of updates to the server, which we call a \textit{flooding attack}.
Notable works in this area typically rely either on strong assumptions that restrict flooding attacks \citep{dahan_weight_2024} or on nearly synchronous algorithms that mitigate the effects of Byzantine clients \citep{yang2021basgd, yang2023buffered}. In practice, the former may offer limited robustness against Byzantine attacks that violate these assumptions, while the latter may yield smaller speedups than asynchronous SGD because it more closely resembles synchronous algorithms.

We aim to bridge this gap through a new algorithm, {\myalgo}.
The key idea is to exponentially decrease the stepsize when gradients are repeatedly received from the same client.
This limits the influence of any individual client on the global model, thereby preventing a malicious client from dominating the overall update.
We present {\myalgo} in Section~\ref{sec:algo} and analyze its convergence in Section~\ref{sec:theory}.
In Theorem~\ref{thm:main}, we establish its Byzantine robustness and show that it recovers the convergence rate of standard asynchronous SGD up to an additive constant.
Our experiments in Section~\ref{sec:expe} demonstrate robustness to a range of Byzantine attacks and show that {\myalgo} is more robust and faster than the existing methods evaluated.
Interestingly, {\myalgo} also outperforms standard asynchronous SGD even in the non-Byzantine setting.

\section{Preliminaries}
\label{sec:preliminaries}

We consider the following optimization problem:
\[
    f(\bm x) \coloneqq \E_{\xi \sim \mathcal{D}}  [F(\bm x; \xi)],    
\]
where $F: \R^d \times \Omega \to \R$ is the loss with respect to a sample $\xi \in \Omega$, and $\gD$ is the given data distribution over the sample space $\Omega$. 

Let $n\in\mathbb N$ denote the total number of clients, and write $[n]:=\{1,\ldots,n\}$ the set of all clients. All clients have access to the loss $F$ and sample from the same dataset $\gD$. We assume that it admits an (unknown) partition $[n]=\mathcal{G} \sqcup \mathcal{B}$ between good and Byzantine clients. Any good client $i\in\mathcal G$ follows the protocol given by the algorithm when sending feedback, while a Byzantine client $i\in\mathcal B$ is allowed to send arbitrary updates to the server. We let the latter be omniscient, i.e., they have access to all computations made by the rest of the good clients and can send the worst updates possible.  We denote the Byzantine ratio by $\delta = |\mathcal{B}|/n$, so that $|\mathcal{G}|=(1-\delta)n$.
In this paper we develop an algorithm with asynchronous updates that is robust in the Byzantine setting described above.

\subsection{Asynchronous SGD}
Synchronous SGD can suffer from substantial delays because each iteration is limited by the slowest client, which motivates the use of asynchronous SGD.
In asynchronous SGD, the server applies each stochastic gradient the moment it arrives. The updated model is then sent to the corresponding client so that the client can start computing the next gradient.

Let $j_t$ be the client whose gradient is applied at time $t$.
For each client $i$, let $\prev(t,i)$ denote the most recent update before $t$ at which client $i$ returned a gradient, and let $\nexxt(t,i)$ denote the first update at or after $t$ at which client $i$ returns a gradient. 
After returning this gradient, client $i$ receives the updated model $\bm{x}^{(\prev(t,i))}$ from the server and starts computing its next gradient from this model.
Therefore, when client $j_t$ returns its next gradient at update $t$, that gradient is evaluated at the stale model $\bm{x}^{(\prev(t,j_t))}$ rather than the current model. 
We call $\tau_t := t - \prev(t,j_t)$ the delay of this gradient.
The update rule of Asynchronous SGD is then written as follows:
\begin{equation} \label{eq:async_SGD}
    \bm x^{(t)} = \bm x^{(t-1)} - \eta \nabla F(\bm x^{(t-\tau_t)} ; \xi_{j_t}^{(t - \tau_t)}).
\end{equation}
Thanks to this update rule, the server does not need to wait for a fresh gradient evaluated at the latest model and can instead immediately update the model using a possibly stale gradient evaluated at $\bm{x}^{(t-\tau_t)}$. Consequently, the optimization process is not bottlenecked by the slowest client, as the server can continue updating the model whenever gradients from other clients arrive.

\subsection{Byzantine-robust Synchronous Algorithm}
Because distributed learning involves many participating clients, some of them may be malicious or faulty. This makes robustness against such Byzantine clients an important requirement for reliable distributed learning. In this section, we briefly describe a representative Byzantine-robust method in the synchronous setting.
At every round $t$, each client $i$ computes a stochastic gradient $\bm g_i^{(t)}$ at the same point $\bm x^{(t)}$, and the server updates the parameter as follows: 
\begin{align*}
    \bm x^{(t+1)} = \bm x^{(t)} - \frac{\eta}{n} \sum_{i=1}^n \bm g_i^{(t)}.
\end{align*}
Since the server simply averages the gradients received from the clients, a Byzantine client can arbitrarily manipulate the parameters of the server.
Byzantine-robust methods replace this average with a \textit{robust aggregator} (e.g., the median)~\citep{blanchard2017krum, karimireddy2021icml}, as is standard in classical robust statistics~\citep{huber2011robust}. 

\begin{definition}[$(\delta, c)$-robust aggregator] 
\label{def:robust_agg}
Suppose that $\bm{g}_1, \dots, \bm{g}_n$ are random vectors of which a subset $\mathcal{G} \subset [n]$ if size at least $|\mathcal{G}|>(1-\delta) n$ and $\delta < 0.5$ and satisfies $\mathbb{E}\left\|\bm{g}_i-\bm{g}_j\right\|^2 \leq \rho^2, \; \forall i,j \in \gG$. Then, the output $\hat{\bm{g}}$ of a Byzantine robust aggregator satisfies:
\begin{align*}
    \mathbb{E}\|\hat{\bm{g}}-\overline{\bm{g}}\|^2 \leq c \delta \rho^2 \quad \text{where} \quad \hat{\bm g} \coloneqq \Agg \left(\bm{g}_1, \dots, \bm{g}_n\right) \; \text{and} \;\; \overline{\bm{g}}:=\frac{1}{|\mathcal{G}|} \sum_{j \in \mathcal{G}} \bm{g}_j.
\end{align*}
\end{definition}

This definition is quite common in the existing literature \citep{karimireddy2021icml,karimireddy2022iclr, gorbunov_variance_2023,yang2024on}.
Intuitively, the role of a robust aggregator $\Agg$ is to estimate the average of the non-Byzantine clients $\overline{\bm{g}}$ with an error of order $\mathcal{O} (\delta)$.
\citet{karimireddy2021icml} first showed that \textit{centered clipping} satisfies this definition.
\citet{karimireddy2022iclr} later extended this result to classical robust aggregators, namely \textit{coordinate-wise median}~\citep{chen2017distributed}, \textit{Krum}~\citep{blanchard2017krum}, and \textit{robust federated averaging} (RFA)~\citep{pillutla2019rfa}, when combined with their proposed pre-aggregation technique, 

\subsection{Byzantine-robust Asynchronous Algorithm}

Combining asynchrony with Byzantine robustness is desirable for distributed learning, as it allows the system to benefit from heterogeneous client speeds while remaining reliable in the presence of malicious or faulty clients. Several Byzantine-robust asynchronous methods have therefore been proposed \citep{yang2023buffered, dahan_weight_2024}.
However, existing theoretical work on Byzantine-robust asynchronous SGD~\citep{yang2023buffered, dahan_weight_2024} tackles this issue by assuming that $\tau_t$ is bounded by a known $\tau_{\mathrm{max}}$ for most clients, or by directly bounding the proportion of Byzantine feedback. These assumptions circumvent the main challenge of the framework, which is precisely that a Byzantine client can make $\tau_t$ arbitrarily large through a flooding attack. 
We propose a Byzantine-robust asynchronous algorithm \myalgo{} that removes these assumptions.

\section{\textsc{Throttle}: Asynchronous SGD with Soft Throttling}
\label{sec:algo}
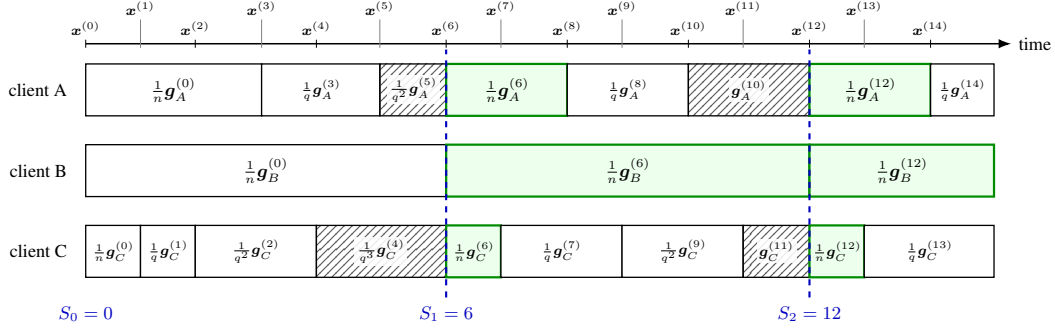
\begin{figure}[t]
    \centering
    \resizebox{\linewidth}{!}{%
    \begin{tikzpicture}[
        font=\small,
        job/.style={draw=black, line width=0.65pt, fill=white},
        restarted/.style={job, draw=green!55!black, line width=1.1pt,
            fill=green!7},
        discarded/.style={job, pattern=north east lines,
            pattern color=black!65},
        sync/.style={blue!75!black, dashed, line width=1.05pt},
        weighted/.style={font=\scriptsize, inner sep=0.25pt},
        discard label/.style={font=\scriptsize, fill=white,
            fill opacity=0.88, text opacity=1, inner sep=1pt}
    ]
        \def\xZero{0.00}
        \def\xOne{0.95}
        \def\xTwo{1.90}
        \def\xThree{3.05}
        \def\xFour{4.00}
        \def\xFive{5.10}
        \def\xSix{6.25}
        \def\xSeven{7.20}
        \def\xEight{8.35}
        \def\xNine{9.30}
        \def\xTen{10.45}
        \def\xEleven{11.40}
        \def\xTwelve{12.55}
        \def\xThirteen{13.50}
        \def\xFourteen{14.65}
        \def\timeEnd{15.75}

        \draw[-{Latex[length=2.2mm]}, line width=0.75pt]
            (\xZero,4.85) -- (16.05,4.85) node[right] {time};
        \foreach \x/\t in {
            \xZero/0,\xTwo/2,\xFour/4,\xSix/6,
            \xEight/8,\xTen/10,\xTwelve/12,\xFourteen/14
        } {
            \draw (\x,4.78) -- (\x,4.92);
            \node[above, font=\scriptsize, inner sep=2pt]
                at (\x,4.92) {$\bm x^{(\t)}$};
        }
        \foreach \x/\t in {
            \xOne/1,\xThree/3,\xFive/5,\xSeven/7,
            \xNine/9,\xEleven/11,\xThirteen/13
        } {
            \draw[black!50] (\x,4.78) -- (\x,5.27);
            \node[above, font=\scriptsize, inner sep=2pt]
                at (\x,5.27) {$\bm x^{(\t)}$};
        }

        \node[left] at (-0.20,4.05) {client A};
        \node[left] at (-0.20,2.65) {client B};
        \node[left] at (-0.20,1.25) {client C};

        \draw[job] (\xZero,3.60) rectangle (\xThree,4.50);
        \node at (1.52,4.05) {$\frac 1n \bm g_A^{(0)}$};
        \draw[job] (\xThree,3.60) rectangle (\xFive,4.50);
        \node[weighted] at (4.08,4.05)
            {$\tfrac{1}{q}\bm g_A^{(3)}$};
        \draw[discarded] (\xFive,3.60) rectangle (\xSix,4.50);
        \node[discard label] at (5.68,4.05) {$\frac{1}{q^2}\bm g_A^{(5)}$};

        \draw[restarted] (\xSix,3.60) rectangle (\xEight,4.50);
        \node at (7.30,4.05) {$\frac 1n \bm g_A^{(6)}$};
        \draw[job] (\xEight,3.60) rectangle (\xTen,4.50);
        \node[weighted] at (9.40,4.05)
            {$\tfrac{1}{q}\bm g_A^{(8)}$};
        \draw[discarded] (\xTen,3.60) rectangle (\xTwelve,4.50);
        \node[discard label] at (11.50,4.05) {$\bm g_A^{(10)}$};

        \draw[restarted] (\xTwelve,3.60) rectangle (\xFourteen,4.50);
        \node at (13.60,4.05) {$\frac 1n \bm g_A^{(12)}$};
        \draw[job] (\xFourteen,3.60) rectangle (\timeEnd,4.50);
        \node[weighted] at (15.20,4.05)
            {$\tfrac{1}{q}\bm g_A^{(14)}$};

        \draw[job] (\xZero,2.20) rectangle (\xSix,3.10);
        \node at (3.12,2.65) {$\frac 1n \bm g_B^{(0)}$};
        \draw[restarted] (\xSix,2.20) rectangle (\xTwelve,3.10);
        \node at (9.40,2.65) {$\frac 1n \bm g_B^{(6)}$};
        \draw[restarted] (\xTwelve,2.20) rectangle (\timeEnd,3.10);
        \node at (14.15,2.65) {$\frac 1n \bm g_B^{(12)}$};

        \draw[job] (\xZero,0.80) rectangle (\xOne,1.70);
        \node[font=\scriptsize] at (0.48,1.25) {$\frac 1n \bm g_C^{(0)}$};
        \draw[job] (\xOne,0.80) rectangle (\xTwo,1.70);
        \node[weighted] at (1.43,1.25)
            {$\tfrac{1}{q}\bm g_C^{(1)}$};
        \draw[job] (\xTwo,0.80) rectangle (\xFour,1.70);
        \node[weighted] at (2.95,1.25)
            {$\tfrac{1}{q^2}\bm g_C^{(2)}$};
        \draw[discarded] (\xFour,0.80) rectangle (\xSix,1.70);
        \node[discard label] at (5.12,1.25) {$\frac{1}{q^3}\bm g_C^{(4)}$};

        \draw[restarted] (\xSix,0.80) rectangle (\xSeven,1.70);
        \node[font=\scriptsize] at (6.73,1.25) {$\frac 1n \bm g_C^{(6)}$};
        \draw[job] (\xSeven,0.80) rectangle (\xNine,1.70);
        \node[weighted] at (8.25,1.25)
            {$\tfrac{1}{q}\bm g_C^{(7)}$};
        \draw[job] (\xNine,0.80) rectangle (\xEleven,1.70);
        \node[weighted] at (10.35,1.25)
            {$\tfrac{1}{q^2}\bm g_C^{(9)}$};
        \draw[discarded] (\xEleven,0.80) rectangle (\xTwelve,1.70);
        \node[discard label] at (11.98,1.25) {$\bm g_C^{(11)}$};

        \draw[restarted] (\xTwelve,0.80) rectangle (\xThirteen,1.70);
        \node[font=\scriptsize] at (13.03,1.25) {$\frac 1n \bm g_C^{(12)}$};
        \draw[job] (\xThirteen,0.80) rectangle (\timeEnd,1.70);
        \node[weighted] at (14.63,1.25)
            {$\tfrac{1}{q}\bm g_C^{(13)}$};

        \draw[sync] (\xSix,0.45) -- (\xSix,4.75);
        \draw[sync] (\xTwelve,0.45) -- (\xTwelve,4.75);
        \node[blue!75!black, below] at (\xSix,0.42) {$S_1=6$};
        \node[blue!75!black, below] at (\xTwelve,0.42) {$S_2=12$};
        \node[blue!75!black, below] at (\xZero,0.42) {$S_0=0$};
    \end{tikzpicture}%
     }
    \caption{Illustration of \textsc{Throttle} with three clients. A block $\bm g_i^{(t)}$ is a gradient computation started at $\bm x^{(t)}$ (its left end), and its coefficient is the weight it receives on arrival, $1/n$ for a first arrival and $1/q^{c}$ for a repeated one.
Dashed lines mark the synchronous clock $S_r$, at which in-flight computations are discarded (hatched blocks) and all clients restart from $\bm x^{(S_r)}$ (green blocks).}
    \label{fig:hard-restart-timeline}
\end{figure}

We introduce {\myalgo}, an algorithm consisting of two key components, \textit{soft throttling} and \textit{hard restart}, which we describe below.

\subsection{Soft Throttling}

The key idea of our approach is to exponentially down-weight repeated updates of the same client by a factor $q$. More precisely, the first update has a factor $\frac1n$, the second $\frac1q$, the next $\frac1{q^2}$, then $\frac1{q^3}$, {\it etc}. Once every client has sent feedback, we say that the round ends. We then reinitialize all factors to $\frac1n$ and start the next round. This way, within any given round, the impact of one client is upper bounded by $\frac1n+\sum_{i=1}^{+\infty}\frac1{q^i}=\frac1n+\frac1{q-1}$, which prevents flooding.

A smaller $q \to 1$ lets fast clients contribute more, while a larger $q \to \infty$ caps their influence more tightly and reduces it to synchronous SGD.
We formally quantify this trade-off in Section~\ref{sec:theory}.

\begin{remark}[On the initial $\frac1n$ factor]
    Adding an initial $1/n$ factor makes the total update with regards to the model at the start of the round $\bm x^{(S_r)}$ have weight $1$. All subsequent updates within the round have weight at most $\frac1q$. This matches the scale of a robust aggregator's output, which simplifies the comparison in our analysis. We also observed that it leads to more stable training curves in practice.
\end{remark}

\subsection{Hard Restart}

At the end of a round, the server broadcasts the current model to all clients, and discard any gradient computation in progress. This ensures that every round starts with all clients on the same model, which allows to robustly aggregate their updates. At the beginning of a round, all clients start their computation at the same point $\bm x^{(S_r)}$, and the added factor $\frac1n$ averages over the updates.
As is standard in Byzantine-robust algorithms, we want to replace this average by a robust aggregator.
The one difference from the synchronous setting is that the first arrivals from each clients do not arrive together.
We therefore need an aggregator that acts on a single input and whose sum coincide with a robust aggregate over time.

\begin{definition}[Local $(c,\delta)$-robust aggregator]
\label{def:local_robust_agg}
Let $\Agg$ be a $(c,\delta)$-robust aggregator that additionally takes a side information $\mathbf{S} \in \sS$ as input.
We call $\Agg$ \textit{local} if there exists a map $A : \mathbb{R}^d \times \sS \to \mathbb{R}^d$ such that, for every $\mathbf{S} \in \sS$ and every $\bm x_1, \dots, \bm x_n \in \mathbb{R}^d$,
\[
    \Agg(\bm x_1, \dots, \bm x_n; \mathbf{S}) = \frac{1}{n} \sum_{i=1}^n \mathbb{A}(\bm x_i; \mathbf{S}).
\]
\end{definition}
 
\begin{example}[Centered clipping]
\label{ex:centered_clipping}
Centered clipping \citep{karimireddy2021icml} clips each input's offset from an anchor $\bm a \in \mathbb{R}^d$ to length at most $\tau > 0$ and then averages:
\[
    \Agg(\bm x_1, \dots, \bm x_n; \bm a, \tau) = \bm a + \frac{1}{n} \sum_{i=1}^n \clip_\tau(\bm x_i - \bm a),
    \qquad
    \clip_\tau(\bm v) \coloneqq \min\Bigl\{1, \frac{\tau}{\|\bm v\|}\Bigr\} \bm v.
\]
Thus, each input shifts the output by at most $\tau/n$. It is local with side information $\mathbf{S} = (\bm a, \tau)$ and $\mathbb{A}(\bm v; \mathbf{S}) \coloneqq \bm a + \clip_\tau(\bm v - \bm a)$. Following \citet{karimireddy2021icml}, our experiments use the previous round's first arrivals to set the anchor:
$
    \bm a_r \coloneqq \Pi_G\Bigl(\frac{1}{n} \sum_{i=1}^n \bm g_i^{(S_{r-1})}\Bigr),
    \bm a_0 \coloneqq \bm 0
$, where $\Pi_G$ is the Euclidean projection onto the ball of radius $G$ centered at the origin.   
\end{example}

Side information is the key to making the decomposition possible.
Without side information, no $(c,\delta)$-robust aggregator of the form $\frac1n\sum_i f(\bm x_i)$ exists (Proposition~\ref{prop:no_side_info}).
Appendix~\ref{app:local_agg} gives further examples and a general construction $\mathbb{A}(\bm v; \mathbf{S})=\bm a+\psi(\bm v-\bm a)$ with bounded $\psi$.

\subsection{{\myalgo}}

\begin{algorithm}[t]
    \caption{\textsc{Throttle}: Asynchronous SGD with Soft Throttling}
    \label{alg:stoch-async-hard-sgd}
    \noindent\underline{\textbf{Server:}}\par
    \begin{algorithmic}[1]

    \State Send $\bm x^{(0)}$ to every client $i\in[n]$ and initialize $S_0 \coloneqq 0$ and side information $\mathbf{S}_0$, $c_i = 0,\; \forall i \in [n]$

    \For{$t = 0, \ldots, T-1$}
        \State Receive update $\bm v^{(t-\tau_t)}_{j_t}$ from some worker $j_t$
        \If{$c_{j_t} = 0$}\Comment{If $j_t$ returns its update for the first time this round%
        }
            \State
                \(
                \bm g^{(t-\tau_t)}_{j_t}
                =
                \mathbb{A}\left(
                \bm v^{(t-\tau_t)}_{j_t}; \mathbf{S}_{r(t)}
                \right)
                \) 
            \Comment{Apply local $(c,\delta)$-robust aggregator}

        \Else\Comment{If server has already seen $j_t$ in this round}
            \State
                \(
                \bm g^{(t-\tau_t)}_{j_t}
                =
                \clip_{\lambda_t} \left(
                \bm v^{(t-\tau_t)}_{j_t}
                \right)
                \)
            \Comment{Compute clipped gradient with radius $\lambda_t$}
        \EndIf 
    \State $\eta_{t} \coloneqq
    \begin{cases}
        \frac{\eta}{n} & \text{if } c_{j_t}(t) = 0, \\
        \frac{\eta}{q^{c_{j_t}(t)}} & \text{otherwise},
    \end{cases}$
            \Comment{Set the \textbf{soft throttling} stepsize (\ref{eq:stepsize})}

        \State 
        $
        \bm x^{(t)}
        =
        \bm x^{(t-1)}
        -
        \eta_{t}\bm g^{(t-\tau_t)}_{j_t}
        $
        \Comment{Perform the update}
        \State $c_{j_t} \gets c_{j_t}+1$
        \Comment{Increment the counter}
        \State Send $\bm x^{(t)}$ to worker $j_t$

        \If{$|\{ i \in [n] \mid c_i \ge 1 \}|  =  n$  }
        \Comment{If all clients has shown up in the round}
            \State
            \(
            S_{r+1} = t
            \) and set $c_i \gets 0$ for all $i \in [n]$.
            \Comment{The next round starts}
            \State Update the side information $\mathbf{S}_{r+1}$
            \Comment{Depends on the aggregator (cf Ex.~\ref{ex:centered_clipping})}
            \State Send $\bm x^{(t)}$ and restart signals to all clients, $c_{i} \gets 0$
            \Comment{\textbf{Hard restart}}
        \EndIf
    \EndFor
    \State Notify all workers to stop
    \end{algorithmic}
    \noindent\underline{\textbf{Good client $i \in \gG$:}}\par
    \begin{algorithmic}
        \Repeat 
            \Repeat
                \State Wait until receiving the parameter $\bm x^{(k)}$ for some $k \in [T]$
                \State Randomly sample $\xi_i^{(k)} \sim \gD$ and then compute the stochastic gradient $\nabla F(\bm x^{(k)}; \xi_i^{(k)})$
                \State Send the stochastic gradient $\nabla F(\bm x^{(k)}; \xi_i^{(k)})$ to the server
            \Until{receive server's notification to restart with $\bm x^{(r(t))}$}
            \State Use restart parameter $\bm x^{(r(t))}$ to compute $\nabla F(\bm x^{(r(t))}; \xi_i^{r(t)})$ with the sample $\xi_i^{r(t)} \sim \mathcal{D}$
        \Until{receive server's notification to stop}
    \end{algorithmic}
\end{algorithm}
Combining these two ideas, soft throttling and hard restart, we propose {\myalgo}, as described in Algorithm~\ref{alg:stoch-async-hard-sgd}.
At time step $t$, let $\bm v^{(t-\tau_t)}_{j_t}$ denote the gradient received from client $j_t$, and let $\eta_t$ denote corresponding stepsize. We define
\begin{equation} \label{eq:stepsize}
\bm v^{(t-\tau_t)}_{j_t} \coloneqq
\begin{cases}
    \nabla F(\bm x^{(t-\tau_t)};\xi_{j_t}^{(t-\tau_t)}) & \text{if } j_t \in \gG, \\
    * & \text{if } j_t \in \gB,
\end{cases}
\qquad
\eta_{t} \coloneqq
\begin{cases}
    \frac{\eta}{n} & \text{if } c_{j_t}(t) = 0, \\
    \frac{\eta}{q^{c_{j_t}(t)}} & \text{otherwise},
\end{cases}
\end{equation}
and $c_i(t) \coloneqq \sum_{s=r(t)}^{t-1} \mathbf{1}_{\{j_s=i\}}$ is the number of gradients that client $i$ has already delivered in the current round.
Once the server receives the gradient from client, the server transforms it to $\bm g^{(t-\tau_t)}_{j_t}$ using either local $(c, \delta)$-robust aggregator or gradient clipping, and updates the parameters, as shown in line 8 of Algorithm~\ref{alg:stoch-async-hard-sgd}.
Then, after the server receives the gradient from all clients, the next round starts and all counters are reinitialized to $0$. We illustrate the algorithm in Figure~\ref{fig:hard-restart-timeline} and provide a notation list in Appendix~\ref{app:notations}.

\section{Theoretical Analysis} \label{sec:theory}

Here we list assumptions we use in our theoretical results.
\begin{assumption}[Unbiased and uniformly bounded stochastic gradients]
\label{ass:stochastic-gradients}
For each $\bm x\in\R^d$, the stochastic gradients from good clients
$\nabla F(\bm x;\xi_i)$, $i\in\mathcal G$, are conditionally i.i.d.,
unbiased, and have uniformly bounded variance:
\[
    \E[\nabla F(\bm x;\xi_i)\mid \bm x]
    =
    \nabla f(\bm x),
    \qquad
    \E\|\nabla F(\bm x;\xi_i)-\nabla f(\bm x)\|^2
    \leq
    \sigma^2 .
\]
\end{assumption}

\begin{assumption}[$L$-smoothness, lower boundedness]
\label{ass:objective}
The objective $f:\R^d\to\R$ is differentiable, lower bounded by $f_\star$, and
$L$-smooth, i.e., for all $\bm x,\bm y\in\R^d$,
\[
    f(\bm y)
    \leq
    f(\bm x)
    + \langle \nabla f(\bm x), \bm y-\bm x\rangle
    + \frac{L}{2}\|\bm y-\bm x\|^2 .
\]
\end{assumption}

\begin{assumption}[Bounded gradient]
    \label{ass:bdd_gradient}
    There exists $G \geq 0$ such that for each $\bm x \in \R^d$ and $\xi \in \Omega$, the stochastic gradient satisfies $$\| \nabla  F(\bm x;\xi) \| \leq G \quad \text{almost surely}. $$
\end{assumption}

The assumptions on smoothness and stochastic gradients are standard in the optimization literature~\citep{bubeck2015convex, nesterov2018lectures}.
The bounded gradient assumption is likewise common in the analysis of asynchronous SGD~\citep{mishchenko_asynchronous_2023, shi_ordered_2025}, as well as in the analysis of clipped gradient descent~\citep{zhang2020improved, Zhang2020Why}, which we employ in non-restart rounds.

Before we state the main theorem, we need to make an additional weak technical assumption on the local robust aggregator $\mathbb{A}$.
\begin{assumption}[Bounded local contributions]
\label{ass:stoch-bounded-local-contributions}
There exists $G' \geq 0$ such that, for every side information $S$ used by the algorithm and every
$\bm v\in\R^d$,
$
    \|\mathbb A(\bm v;S)\|\leq G'.
$
\end{assumption}

\begin{remark}[Centered-clipping example]
\label{rem:stoch-centered-clipping-bounded}
Suppose that the side information is
$S=(\bm a_r,2G)$, where the projected anchor in Example~\ref{ex:centered_clipping} satisfies $\|\bm a_r\|\leq G$.  For the
centered-clipping map
\(
    \mathbb A(\bm v;\mathbf{S})
    =
    \bm a_r+\clip_{2G}(\bm v-\bm a_r),
\)
we have $\|\mathbb A(\bm v;S)\|\leq G+2G=3G$ for every input, including
a Byzantine input, verifying Assumption~\ref{ass:stoch-bounded-local-contributions}.
Moreover, Assumption~\ref{ass:bdd_gradient} implies that an honest input satisfies $\|\bm v\|\leq G$, and hence
$\|\bm v-\bm a_r\|\leq2G$, so centered clipping therefore leaves every honest input unchanged.  
Similar assumptions are used in \citep[Assumption~2.3]{malinovsky_byzantine_2024}.
\end{remark}

Now, we present the main theorem of this paper. We analyze the convergence in terms of effective steps, which quantify the number of steps normalized by the decreasing stepsizes. In the synchronous setting, that is if every client submits the feedback at the same time, the number of effective steps $T_{\mathrm{eff}}$ is exactly the number of feedbacks ($T$ in the synchronous setting). Moreover, it closely corresponds to the number $R$ of rounds as $R\leq T_{\mathrm{eff}} = R+\sum_{r=1}^{R}\sum_{i=1}^{n} \frac{1-q^{-(c_{i,r}-1)}}{q-1}$.

\begin{theorem}[Non-convex convergence rate]
\label{thm:main}
Let $T\geq2$ and $q>1$.  Suppose that Assumptions
\ref{ass:stochastic-gradients},~\ref{ass:objective}, and
\ref{ass:bdd_gradient}, as well as
Assumption~\ref{ass:stoch-bounded-local-contributions} with $G' \leq 3G$ hold, together with
the local $(c,\delta)$-robust aggregator $\mathbb{A}$, and use\footnote{Under Assumption~\ref{ass:bdd_gradient}, honest stochastic gradients satisfy $\|\nabla F(\bm x; \xi)\| \le G$ almost surely. With $\lambda_t = G$, the clipping operator therefore acts as the identity on honest gradients, which simplifies the analysis. The same assumption also gives $\sigma^2 \le G^2$. In our experiments, we use a smaller threshold ($\lambda_t = 1$), for which clipping is typically active. Extending the analysis to this regime is left for future work.} $\lambda_t=G$.
Then, there exists a small enough stepsize $\eta > 0$ (detailed in Appendix~\ref{app:stochastic_hard_reset}) such that, 
\begin{align*}
    \E\|\nabla f(\bm x^{\mathrm{out}})\|^2
    =\gO\Bigg(
        &\left(
            c\delta\sigma^2
            +\frac{G^2|\gB|^2}{(q-1)^2}
        \right)
        \theta_T + \sqrt{\frac{L\sigma^2\Delta}{T_{\mathrm{eff}}}
            \left(1-\theta_T+\frac{\theta_T}{|\gG|}\right)}\\
        &+\left(\frac{L\Delta}{T_{\mathrm{eff}}}\right)^{2/3}
        \left[
            (1-\theta_T)\left(
                (n-1)^2G^2
                +
                c\delta\sigma^2
                +\frac{G^2|\gB|^2}{(q-1)^2}
            \right)
        \right]^{1/3} + \frac{L\Delta}{T_{\mathrm{eff}}}
    \Bigg),
\end{align*}
where $\Delta\coloneqq f(\bm x^{(0)})-f_*$ is the initial distance, $T_{\mathrm{eff}}\coloneqq 1 +\sum_{t=1}^{T-1}\hat\eta_t/\eta $ is the number of effective time steps, $N_T \coloneqq 1+\sum_{t=1}^{T-1}\indi_{\{t=r(t)\}}$ is the number of rounds, $\theta_T \coloneqq \frac{N_T}{T_{\mathrm{eff}}}$ is the restart ratio, and $\bm x^{(\mathrm{out})}$ is the weighted average of $\bm x^{(0)}, \ldots, \bm x^{(T-1)}$ defined in (\ref{eq:weighted_average_in_main_thm}).
\end{theorem}
\begin{remark}[Synchronous limit $q \to \infty$]~\label{rem:sync}
    As $q \to \infty$, an effective step corresponds to a round.  
    Thus, $T_{\mathrm{eff}}\to R$ and
    $\theta_T\to1$, where $R$ is the total number of rounds
    including the initial round.  The convergence rate simplifies to
    \[
        \E\|\nabla f(\bm x^{\mathrm{out}})\|^2
        =\gO\left(
            \sqrt{\frac{L\sigma^2\Delta}{|\gG|R}}
            +\frac{L\Delta}{R}
            +c\delta\sigma^2
        \right),
    \]
    which recovers the standard rate of synchronous Byzantine-robust SGD.
    The non-vanishing term $\gO(c\delta\sigma^2)$ is known to be unavoidable for any $(c,\delta)$-robust aggregator applied only to the current stochastic gradients, and it can be removed by momentum or variance reduction~\citep{karimireddy2021icml, gorbunov_variance_2023}.
    For finite $q$, however, the dominant error is the asynchrony-induced term $G^2|\gB|^2/(q-1)^2$, which persists even with full gradients ($\sigma = 0$). Note also that $\sigma^2 \le G^2$ in the non-clipping regime.
    Hence, momentum or variance reduction alone cannot remove the main source of error in the asynchronous setting, and combining momentum with stale gradients is itself nontrivial~\citep{shi_ordered_2025}.
    We leave reducing both terms to future work.
\end{remark}

\begin{remark}[Non-Byzantine asynchronous limit $|\gB|=0$, $q \to 1$]~\label{rem:async}
    If there are no Byzantine clients and $q \to 1$, then we have $T_{\mathrm{eff}} \to T$ and the convergence rate simplifies to
    \[
        \E \|\nabla f(\bm x^{\mathrm{out}})\|^2
        = \gO\left(
            \sqrt{\frac{L\sigma^2\Delta}{T}}
            +\left(\frac{L\Delta}{T}\right)^{2/3}
            \left[(n-1)^2G^2\right]^{1/3}
            +\frac{L\Delta}{T}
        \right),
    \]
    which is a similar convergence rate (up to the $T^{-1}$ term due to normalizing rounds gradients by $1/n$) to asynchronous SGD without Byzantine clients~\citep[Theorem 2]{mishchenko_asynchronous_2023}.
\end{remark}

\section{Related Work} \label{sec:related_work}
\paragraph{Asynchronous SGD. }
Asynchronous optimization has a long history dating back to the 1970s, when \citet{baudet1978asynchronous} studied asynchronous fixed-point iterations. Subsequently, \citet{tsitsiklis1986distributed} established the convergence of asynchronous stochastic gradient methods in distributed settings.
In the 2010s, attention turned to asynchronous variants of SGD. A notable example is Hogwild!~\citep{recht2011hogwild}, a lock-free scheme that has since been widely adopted. Classical analyses of asynchronous SGD~\citep{agarwa2012distributed, lian2015asynchronous, mania2017perturbed, stich2020error} rely on a bounded delay assumption, so their convergence rates depend on the maximum delay and are often overly pessimistic.
\citet{koloskova_sharper_2022} and \citet{mishchenko_asynchronous_2023} concurrently showed that asynchronous SGD attains convergence rates that do not depend on the maximum delay, significantly improving upon prior analyses.
Since then, a growing body of work has built upon these results.
For instance, \citet{shi_ordered_2025} designed a momentum scheme suited to asynchronous SGD. \citet{islamov24asgrad} developed a unified framework that covers many asynchronous SGD variants, including FedBuff~\citep{nguyen2022federated}.
More broadly, asynchronous methods have been studied under data, computation, and communication heterogeneity~\citep{rennala23, shadowheart24, ringmaster25}, and extended to finite-sum settings~\citep{freya24}.

\paragraph{Byzantine-robust Learning. }
Classical Byzantine-robust methods in the synchronous setting replace mean aggregation with a robust aggregator, such as coordinate-wise median and trimmed mean~\citep{yin2018byzantine}, Krum~\citep{blanchard2017krum}, geometric median (RFA)~\citep{pillutla2019rfa}, MDA~\citep{el2020genuinely}, and Bulyan~\citep{guerraoui2018hidden}.
However, due to the noise inherent in stochastic gradients, attackers can craft malicious updates that are statistically indistinguishable from honest ones and yet prevent convergence, as in the ALIE~\citep{baruch2019nips} and IPM~\citep{xie2019icml} attacks.
Another line of work complements robust aggregation with momentum~\citep{karimireddy2021icml, el2021distributed, farhadkhani2022icml} or variance reduction~\citep{gorbunov_variance_2023}, and the optimality of such approaches has been studied in~\citep{yin2018byzantine, zhu2023aistats, shi2025optimal}.
These results, however, assume synchronous updates. 
Incorporating momentum under asynchrony is nontrivial even in the absence of Byzantine workers, since it accumulates stale gradients~\citep{shi_ordered_2025}.
Byzantine robustness has further been studied under more practical conditions, including data heterogeneity~\citep{karimireddy2022iclr, allouah_fixing_2023, allouah2023nips}, heterogeneous client participation~\citep{allouah_byzantine-robust_2024, malinovsky_byzantine_2024, otsuka2026delayed}, and communication compression~\citep{gorbunov_variance_2023, rammal_communication_2024}.

\paragraph{Asynchronous Byzantine-robust Learning. } 
Combining asynchrony with Byzantine robustness is challenging because asynchronous SGD applies each stochastic gradient immediately upon arrival, leaving no natural averaging step to replace with a robust aggregator.
One line of work addresses this by partially \textit{synchronizing} the updates: BASGD and BASGDm~\citep{yang2021basgd, yang2023buffered} aggregate over server-side buffers, and \citet{dahan_weight_2024} aggregate each incoming gradient with the latest gradients of the other clients using a weighting scheme that mitigates staleness.
Another line of work filters incoming gradients using additional information.
Zeno++~\citep{xie2020zenopp} and AFLGuard~\citep{fang2022aflguard} rely on a trusted dataset held by the server, an assumption that may undermine the computational efficiency of distributed learning. 
While Kardam~\citep{kardam2018} filters updates based on an empirical estimate of the smoothness constant, but may reject honest gradients along with malicious ones~\citep{xie2020zenopp}.

\section{Experiments}
\label{sec:expe}
Our experiments show the benefits of {\myalgo} both with and without Byzantine clients. In the Byzantine setting, it remained competitive under most standard attacks and maintains high accuracy under flooding attacks that severely degrade several existing methods. In the non-Byzantine setting, it achieved lower optimization error than minibatch SGD, asynchronous SGD, and the other evaluated baselines under the same gradient computation budget.

\subsection{{\myalgo} is Byzantine-robust}
\begin{figure}[t]
    \centering
    \includegraphics[width=1\linewidth]{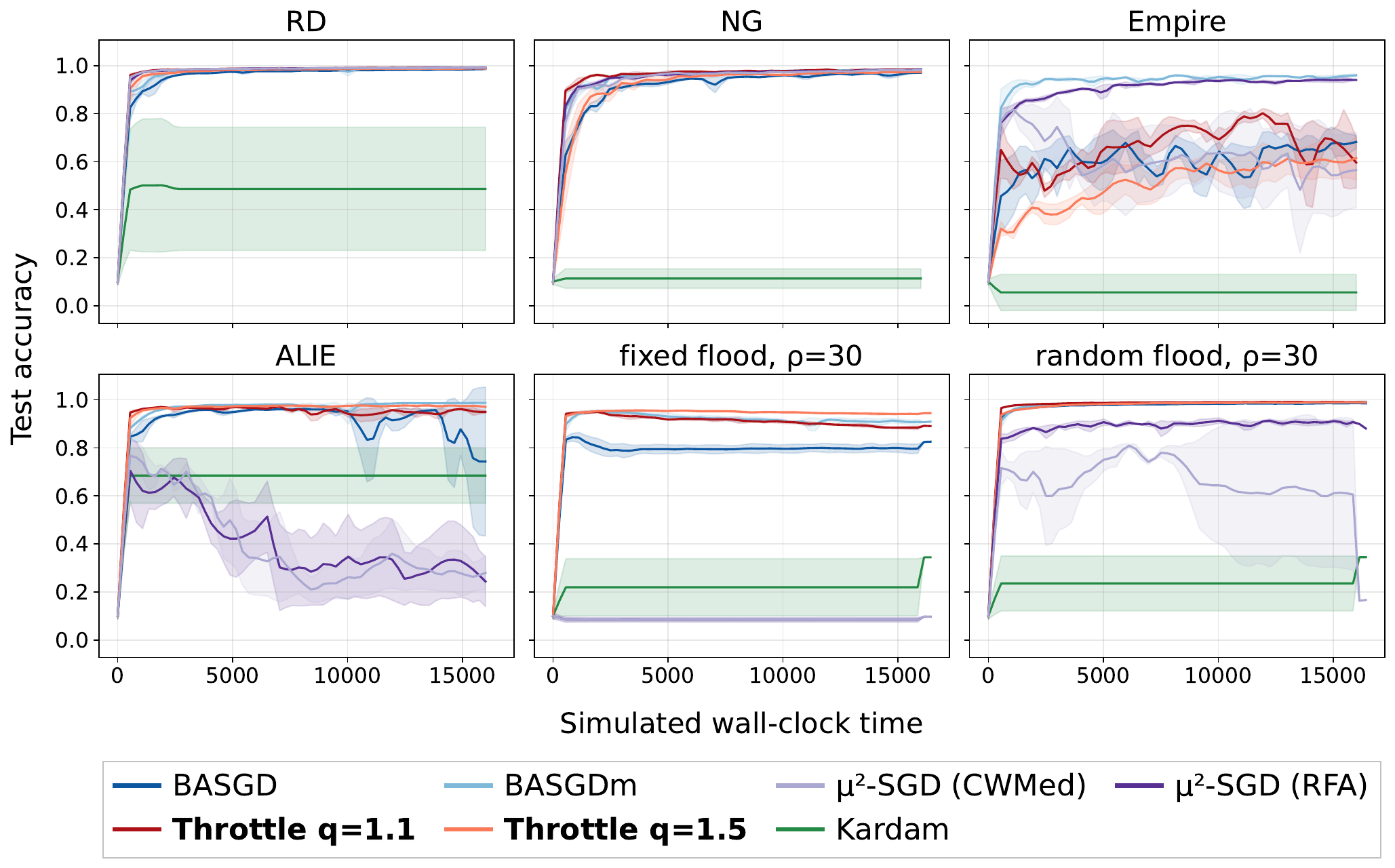}
    \vspace{-2.0em}
    \caption{Test accuracy on MNIST with $n=25$ clients, $|\gB| = 5$ Byzantine clients.
    Standard attacks (RD,NG,Empire,ALIE) under the scheduler based on the codebase\protect\footnotemark[\getrefnumber{fn:wfr}], 
    in which Byzantine updates are at most $1/3$ of all gradients at any time.
    Flooding attacks under using a fixed vector/random vectors. The higher $\rho$, the more frequently the Byzantine clients send their updates. \textsc{Throttle} keeps a higher accuracy overall. Figure~\ref{fig:appendix-exp-figure} also shows the results with other choices of $\rho=1,3,10$.
    }
    \label{fig:byzantine-exp}
\end{figure}

We trained convolutional neural networks using our optimizer and other existing asynchronous SGD algorithms on MNIST with $n=25$ clients and $|\gB| = 5$ Byzantine clients.
We compared \textsc{Throttle} against four standard attacks RD (Random Disturbance attack), NG (Negative Gradient attack~\citep{yang2023buffered}, Empire~\citep{xie_fall_2019}, ALIE~\citep{baruch2019nips} and eight flooding attacks (Figure~\ref{fig:byzantine-exp} and Figure~\ref{fig:appendix-exp-figure}).\footnote{For ALIE and Empire, we used the latest gradients from each honest client as the information for them.} We compared our proposed method with Kardam \citep{kardam2018},
BASGD and BASGDm \citep{yang2021basgd,yang2023buffered}, and Asynchronous Robust $\mu^2$-SGD \citep{dahan_weight_2024} based on their codebase.~\footnote{\url{https://github.com/dahan198/asynchronous-fault-tolerant-ml}\label{fn:wfr}}
Following \citet{dahan_weight_2024}, we simulate asynchronous training as a sequence of updates. 
For the standard attacks in Figure~\ref{fig:byzantine-exp} every third arrival is a Byzantine update and the arriving client is drawn with probability proportional to its index for other updates, as in \citet{dahan_weight_2024}.
In the flooding experiments, client $i$ is an independent Poisson process with
rate $r_i\propto i$ and the Byzantine rates are multiplied by $\rho$, giving a Byzantine fraction of
updates $\rho/(2+\rho)$ ($1/3$ to $0.94$ for $\rho=1$ to $30$). 
Appendix~\ref{app:experiments} gives additional details including all hyperparameter grids, attack details, and asynchrony.

\paragraph{\myalgo{} is robust and outperforms existing algorithms in almost all type of attacks.} 

\myalgo{} outperformed BASGDm under fixed flooding even with increased $\rho$, and outperformed $\mu^2$-SGD with RFA under both ALIE and flooding attacks. The exception was Empire, where \myalgo{} achieved $59.6\%$ and $61.5\%$, compared with $96.0\%$ for BASGDm and $94.0\%$ for $\mu^2$-SGD with RFA. Since both baselines used momentum, we expect that adding momentum to \myalgo{}, together with tuning the clipping radius, would narrow this gap. We leave this to future work.

\begin{figure}[t]
    \centering
    \includegraphics[width=1\linewidth]{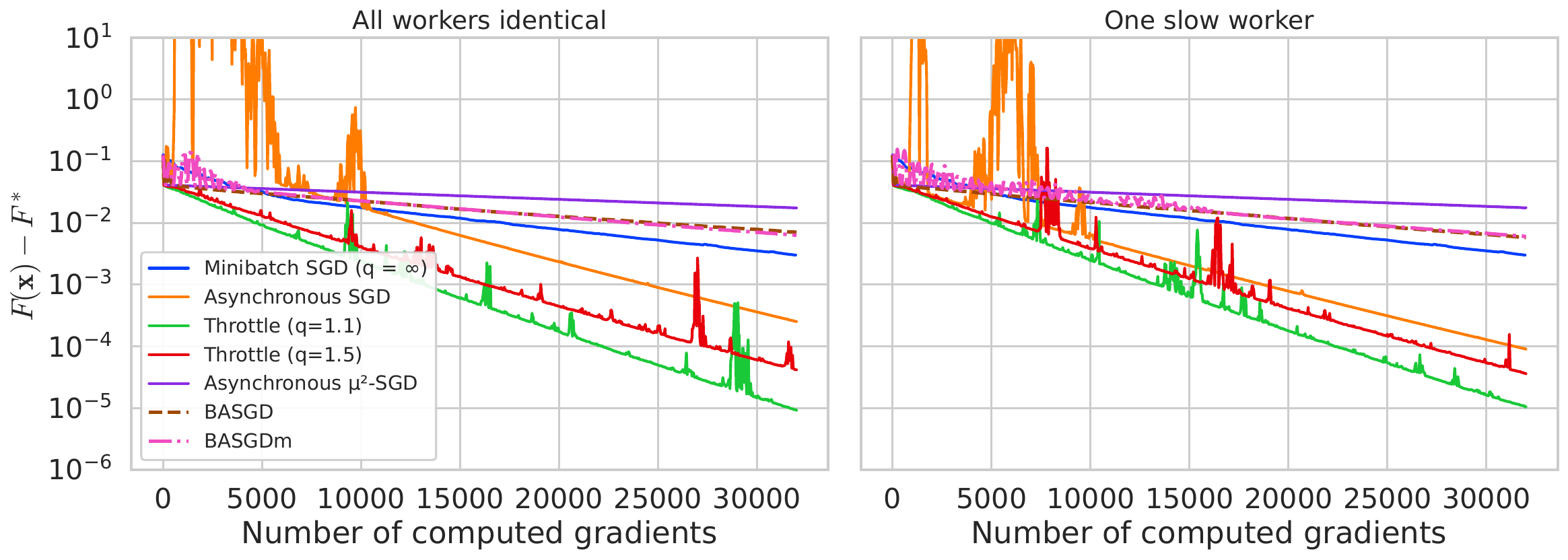}
    \vspace{-2.0em}
    \caption{We ran experiments without Byzantine clients on a simple least-squares problem with random data and tuned all stepsizes for each methods under pure asynchrony using Ray\protect\footnotemark[\getrefnumber{fn:ray}] framework.
    We averaged over three seeds.
    \textit{Left:} all clients compute gradients at similar speeds. \textit{Right:} one worker sleeps for $100\times$ its compute time after every gradient computation. 
    \textbf{\textsc{Throttle} (green) achieves faster convergence} because the \textit{soft throttling} mechanism allows us to take slightly larger stepsize $\eta$ than asynchronous SGD.  }
    \label{fig:no-byzantine-async-throttling}
\end{figure}

\subsection{{\myalgo} Outperforms standard Asynchronous SGD}

Following \citet{mishchenko_asynchronous_2023}, we study least-squares optimization with $40$ asynchronous Ray~\footnote{\url{https://docs.ray.io/en/latest/ray-core/examples/plot_parameter_server.html\#asynchronous-parameter-server-training}\label{fn:ray}} workers and no Byzantine clients (Figure~\ref{fig:no-byzantine-async-throttling}).
We compare \myalgo{} with $q\in\{1.1,1.5\}$ against minibatch SGD ($q=\infty$), asynchronous SGD, $\mu^2$-SGD, BASGD, and BASGDm\footnote{For the latter three methods, we replace robust aggregation with averaging while preserving any weighting scheme used by the original method. Gradient clipping was also disabled $\lambda_t = \infty$ for \myalgo{}.}.
We tuned the stepsizes separately for each method. All methods use a budget of $32{,}000$ computed gradients,
including those discarded at hard restarts by \myalgo{}.
We plot the optimality gap against this count.
The right panel introduces one straggler that sleeps for $100\times$ its gradient computation time after each computation.
Appendix~\ref{app:quadratic} provides further details.

\paragraph{Effectiveness of soft throttling.}
With a tuned stepsize \textsc{Throttle} performs significantly better than its two limits. 
Among the tested values, $q=1.1$ achieves the lowest final optimality gap in both worker configurations.
Although $q=1.5$ converges more slowly per computed gradient, it still outperforms asynchronous SGD and minibatch SGD.
Asynchronous $\mu^2$-SGD, BASGD, and BASGDm did not utilize asynchrony into faster convergence. All three existing methods had slower convergence than minibatch SGD even in the absence of Byzantine clients.

\paragraph{Soft throttling remains effective with a straggler.}
Introducing a worker that sleeps for $100\times$ its computation time
has little effect on the convergence of \myalgo{} per computed gradient
(Figure~\ref{fig:no-byzantine-async-throttling}, right).
Even at $q=1.5$, convergence per computed gradient is largely preserved
relative to the setting without a straggler, although convergence
becomes less stable.

\section{Conclusion}
We introduce soft throttling, an exponential down-weighting of updates by the speed of each client.
We show that the resulting algorithm, \textsc{Throttle}, is a generalization of both synchronous Byzantine-robust SGD ($q \to \infty$), and (almost) asynchronous SGD ($q\to1$) without Byzantine clients.
The convergence rate of \textsc{Throttle} matches the known rate at each limit and holds against flooding attacks without any bound on the Byzantine update ratio. 
Our experiments demonstrate robustness to attacks and show that soft throttling accelerates convergence even in the absence of Byzantine clients.

\section*{Acknowledgements}
Kaoru Otsuka was supported by MEXT Supporting Pioneering Research through AI for 1,000 Discovery challenges Program (SPReAD) Japan Grant Number JPMXP1726264354.
Maxime Meyer was supported by the National Research Foundation, Singapore under its AI Singapore Programme (AISG Award No: AISG3-PhD-2026-01-068T).
Makoto Yamada was partly supported by JSPS KAKENHI Grant Number 24K03004 and by JST ASPIRE JPMJAP2302.

\section*{AI Use Statement}
We used generative AI tools (large language models and LLM-based coding agents) in this work as follows.

\textbf{Theory.} All definitions, main theorem statements, proof strategies, and the proposed algorithm were developed by the authors. LLMs were used to check proofs for errors, correct minor mistakes, verify numerical constants, refine notation, and convert the authors' arguments into polished formal LaTeX.

\textbf{Experiments.} LLM-based coding agents were used to implement methods, design and refine experimental pipelines, and adapt publicly available code from prior works. LLMs also provided limited assistance with interpreting results and qualitative analysis, including cross-checking generated plots against the intended algorithm. 

\textbf{Writing and figures.} We used generative AI tools for drafting and editing parts of the manuscript based on the authors' initial drafts, and for producing the figures as TikZ code from the authors' hand-drawn illustrations.

\textbf{Exclusions.} We did not use generative AI to generate synthetic datasets or to clean or reformat datasets, propose or refine hypotheses, develop the theoretical model or conceptual framework, or propose research ideas.

All LLM-assisted work has been reviewed and tested by the authors.
Proofs were independently checked after LLM-assisted revision, and all LLM-assisted text was verified. 
We take full responsibility for the final content of this work.

\bibliography{references}

\begin{thebibliography}{60}
\providecommand{\natexlab}[1]{#1}
\providecommand{\url}[1]{\texttt{#1}}
\expandafter\ifx\csname urlstyle\endcsname\relax
  \providecommand{\doi}[1]{doi: #1}\else
  \providecommand{\doi}{doi: \begingroup \urlstyle{rm}\Url}\fi

\bibitem[Agarwal \& Duchi(2012)Agarwal and Duchi]{agarwa2012distributed}
Alekh Agarwal and John~C. Duchi.
\newblock Distributed delayed stochastic optimization.
\newblock In \emph{IEEE Conference on Decision and Control}, 2012.

\bibitem[Allouah et~al.(2023{\natexlab{a}})Allouah, Farhadkhani, Guerraoui,
  Gupta, Pinot, and Stephan]{allouah_fixing_2023}
Youssef Allouah, Sadegh Farhadkhani, Rachid Guerraoui, Nirupam Gupta, Rafael
  Pinot, and John Stephan.
\newblock Fixing by mixing: {A} recipe for optimal byzantine {ML} under
  heterogeneity.
\newblock In \emph{International Conference on Artificial Intelligence and
  Statistics}, 2023{\natexlab{a}}.

\bibitem[Allouah et~al.(2023{\natexlab{b}})Allouah, Guerraoui, Gupta, Pinot,
  and Rizk]{allouah2023nips}
Youssef Allouah, Rachid Guerraoui, Nirupam Gupta, Rafael Pinot, and Geovani
  Rizk.
\newblock Robust distributed learning: Tight error bounds and breakdown point
  under data heterogeneity.
\newblock In \emph{Advances in Neural Information Processing Systems},
  2023{\natexlab{b}}.

\bibitem[Allouah et~al.(2024)Allouah, Farhadkhani, Guerraoui, Gupta, Pinot,
  Rizk, and Voitovych]{allouah_byzantine-robust_2024}
Youssef Allouah, Sadegh Farhadkhani, Rachid Guerraoui, Nirupam Gupta, Rafael
  Pinot, Geovani Rizk, and Sasha Voitovych.
\newblock Byzantine-robust federated learning: Impact of client subsampling and
  local updates.
\newblock In \emph{International Conference on Machine Learning}, 2024.

\bibitem[Baruch et~al.(2019)Baruch, Baruch, and Goldberg]{baruch2019nips}
Gilad Baruch, Moran Baruch, and Yoav Goldberg.
\newblock A little is enough: {C}ircumventing defenses for distributed
  learning.
\newblock In \emph{Advances in Neural Information Processing Systems}, 2019.

\bibitem[Baudet(1978)]{baudet1978asynchronous}
Gerard~M Baudet.
\newblock Asynchronous iterative methods for multiprocessors.
\newblock In \emph{Journal of the ACM}, 1978.

\bibitem[Blanchard et~al.(2017)Blanchard, Mhamdi, Guerraoui, and
  Stainer]{blanchard2017krum}
Peva Blanchard, El~Mahdi~El Mhamdi, Rachid Guerraoui, and Julien Stainer.
\newblock Machine learning with adversaries: Byzantine tolerant gradient
  descent.
\newblock In \emph{Advances in Neural Information Processing Systems}, 2017.

\bibitem[Bubeck(2015)]{bubeck2015convex}
S{\'e}bastien Bubeck.
\newblock Convex optimization: Algorithms and complexity.
\newblock \emph{Foundations and trends in Machine Learning}, 8\penalty0
  (3-4):\penalty0 231--357, 2015.

\bibitem[Chen et~al.(2016)Chen, Monga, Bengio, and
  Jozefowicz]{chen2016revisiting}
Jianmin Chen, Rajat Monga, Samy Bengio, and Rafal Jozefowicz.
\newblock Revisiting distributed synchronous sgd.
\newblock In \emph{International Conference on Learning Representations
  Workshop Track}, 2016.

\bibitem[Chen et~al.(2017)Chen, Su, and Xu]{chen2017distributed}
Yudong Chen, Lili Su, and Jiaming Xu.
\newblock Distributed statistical machine learning in adversarial settings:
  Byzantine gradient descent.
\newblock \emph{Proceedings of the ACM on Measurement and Analysis of Computing
  Systems}, 2017.

\bibitem[Dahan \& Levy(2024)Dahan and Levy]{dahan_weight_2024}
Tehila Dahan and Kfir~Y. Levy.
\newblock Weight for robustness: {A} comprehensive approach towards optimal
  fault-tolerant asynchronous {ML}.
\newblock In \emph{Advances in Neural Information Processing Systems}, 2024.

\bibitem[Damaskinos et~al.(2018)Damaskinos, Guerraoui, Patra, Taziki,
  et~al.]{kardam2018}
Georgios Damaskinos, Rachid Guerraoui, Rhicheek Patra, Mahsa Taziki, et~al.
\newblock Asynchronous byzantine machine learning (the case of {SGD}).
\newblock In \emph{International Conference on Machine Learning}, 2018.

\bibitem[El-Mhamdi et~al.(2020)El-Mhamdi, Guerraoui, Guirguis, Hoang, and
  Rouault]{el2020genuinely}
El-Mahdi El-Mhamdi, Rachid Guerraoui, Arsany Guirguis, L{\^e}~Nguy{\^e}n Hoang,
  and S{\'e}bastien Rouault.
\newblock Genuinely distributed byzantine machine learning.
\newblock In \emph{Symposium on Principles of Distributed Computing}, 2020.

\bibitem[Fang et~al.(2022)Fang, Liu, Gong, and Bentley]{fang2022aflguard}
Minghong Fang, Jia Liu, Neil~Zhenqiang Gong, and Elizabeth~S. Bentley.
\newblock Aflguard: {B}yzantine-robust asynchronous federated learning.
\newblock In \emph{Annual Computer Security Applications Conference}, 2022.

\bibitem[Farhadkhani et~al.(2022)Farhadkhani, Guerraoui, Gupta, Pinot, and
  Stephan]{farhadkhani2022icml}
Sadegh Farhadkhani, Rachid Guerraoui, Nirupam Gupta, Rafael Pinot, and John
  Stephan.
\newblock Byzantine machine learning made easy by resilient averaging of
  momentums.
\newblock In \emph{International Conference on Machine Learning}, 2022.

\bibitem[Gorbunov et~al.(2023)Gorbunov, Horv{\'{a}}th, Richt{\'{a}}rik, and
  Gidel]{gorbunov_variance_2023}
Eduard Gorbunov, Samuel Horv{\'{a}}th, Peter Richt{\'{a}}rik, and Gauthier
  Gidel.
\newblock Variance reduction is an antidote to byzantines: Better rates, weaker
  assumptions and communication compression as a cherry on the top.
\newblock In \emph{International Conference on Learning Representations}, 2023.

\bibitem[Horvath et~al.(2021)Horvath, Laskaridis, Almeida, Leontiadis,
  Venieris, and Lane]{horvath2021fjord}
Samuel Horvath, Stefanos Laskaridis, Mario Almeida, Ilias Leontiadis, Stylianos
  Venieris, and Nicholas Lane.
\newblock Fjord: {F}air and accurate federated learning under heterogeneous
  targets with ordered dropout.
\newblock In \emph{Advances in Neural Information Processing Systems}, 2021.

\bibitem[Huber(2005)]{huber2011robust}
Peter~J. Huber.
\newblock Robust statistics.
\newblock In \emph{Wiley Series in Probability and Statistics}, 2005.

\bibitem[Islamov et~al.(2024)Islamov, Safaryan, and Alistarh]{islamov24asgrad}
Rustem Islamov, Mher Safaryan, and Dan Alistarh.
\newblock {AsGrad}: A sharp unified analysis of asynchronous-{SGD} algorithms.
\newblock In \emph{International Conference on Artificial Intelligence and
  Statistics}, 2024.

\bibitem[Kairouz et~al.(2021)Kairouz, McMahan, Avent, Bellet, Bennis, Bhagoji,
  Bonawitz, Charles, Cormode, Cummings, D'Oliveira, Eichner, Rouayheb, Evans,
  Gardner, Garrett, Gasc{\'{o}}n, Ghazi, Gibbons, Gruteser, Harchaoui, He, He,
  Huo, Hutchinson, Hsu, Jaggi, Javidi, Joshi, Khodak, Kone{\v{c}}n{\'y},
  Korolova, Koushanfar, Koyejo, Lepoint, Liu, Mittal, Mohri, Nock,
  {\"{O}}zg{\"{u}}r, Pagh, Qi, Ramage, Raskar, Raykova, Song, Song, Stich, Sun,
  Suresh, Tram{\`{e}}r, Vepakomma, Wang, Xiong, Xu, Yang, Yu, Yu, and
  Zhao]{kairouz2021advances}
Peter Kairouz, H.~Brendan McMahan, Brendan Avent, Aur{\'{e}}lien Bellet, Mehdi
  Bennis, Arjun~Nitin Bhagoji, Kallista~A. Bonawitz, Zachary Charles, Graham
  Cormode, Rachel Cummings, Rafael G.~L. D'Oliveira, Hubert Eichner, Salim~El
  Rouayheb, David Evans, Josh Gardner, Zachary Garrett, Adri{\`{a}}
  Gasc{\'{o}}n, Badih Ghazi, Phillip~B. Gibbons, Marco Gruteser, Za{\"{\i}}d
  Harchaoui, Chaoyang He, Lie He, Zhouyuan Huo, Ben Hutchinson, Justin Hsu,
  Martin Jaggi, Tara Javidi, Gauri Joshi, Mikhail Khodak, Jakub
  Kone{\v{c}}n{\'y}, Aleksandra Korolova, Farinaz Koushanfar, Sanmi Koyejo,
  Tancr{\`{e}}de Lepoint, Yang Liu, Prateek Mittal, Mehryar Mohri, Richard
  Nock, Ayfer {\"{O}}zg{\"{u}}r, Rasmus Pagh, Hang Qi, Daniel Ramage, Ramesh
  Raskar, Mariana Raykova, Dawn Song, Weikang Song, Sebastian~U. Stich, Ziteng
  Sun, Ananda~Theertha Suresh, Florian Tram{\`{e}}r, Praneeth Vepakomma, Jianyu
  Wang, Li~Xiong, Zheng Xu, Qiang Yang, Felix~X. Yu, Han Yu, and Sen Zhao.
\newblock Advances and open problems in federated learning.
\newblock In \emph{Foundations and Trends in Machine Learning}, 2021.

\bibitem[Karimireddy et~al.(2021)Karimireddy, He, and
  Jaggi]{karimireddy2021icml}
Sai~Praneeth Karimireddy, Lie He, and Martin Jaggi.
\newblock Learning from history for byzantine robust optimization.
\newblock In \emph{International Conference on Machine Learning}, 2021.

\bibitem[Karimireddy et~al.(2022)Karimireddy, He, and
  Jaggi]{karimireddy2022iclr}
Sai~Praneeth Karimireddy, Lie He, and Martin Jaggi.
\newblock Byzantine-robust learning on heterogeneous datasets via bucketing.
\newblock In \emph{International Conference on Learning Representations}, 2022.

\bibitem[Koloskova et~al.(2020)Koloskova, Loizou, Boreiri, Jaggi, and
  Stich]{koloskova2020icml}
Anastasia Koloskova, Nicolas Loizou, Sadra Boreiri, Martin Jaggi, and
  Sebastian~U. Stich.
\newblock A unified theory of decentralized {SGD} with changing topology and
  local updates.
\newblock In \emph{International Conference on Machine Learning}, 2020.

\bibitem[Koloskova et~al.(2022)Koloskova, Stich, and
  Jaggi]{koloskova_sharper_2022}
Anastasia Koloskova, Sebastian~U. Stich, and Martin Jaggi.
\newblock Sharper convergence guarantees for asynchronous {SGD} for distributed
  and federated learning.
\newblock In \emph{Advances in Neural Information Processing Systems}, 2022.

\bibitem[Koloskova et~al.(2024)Koloskova, Doikov, Stich, and
  Jaggi]{koloskova2024on}
Anastasia Koloskova, Nikita Doikov, Sebastian~U Stich, and Martin Jaggi.
\newblock On convergence of incremental gradient for non-convex smooth
  functions.
\newblock In \emph{International Conference on Machine Learning}, 2024.

\bibitem[Koloskova et~al.(2023)Koloskova, McKenna, Charles, Rush, and
  McMahan]{koloskova2023gradient}
Anastasiia Koloskova, Ryan McKenna, Zachary Charles, John Rush, and H~Brendan
  McMahan.
\newblock Gradient descent with linearly correlated noise: Theory and
  applications to differential privacy.
\newblock In \emph{Advances in Neural Information Processing Systems}, 2023.

\bibitem[Lamport et~al.(2019)Lamport, Shostak, and
  Pease]{lamport_byzantine_2019}
Leslie Lamport, Robert Shostak, and Marshall Pease.
\newblock \emph{The Byzantine generals problem}.
\newblock Association for Computing Machinery (ACM) / ACM Books, 10 2019.
\newblock ISBN 9781450372701.
\newblock \doi{10.1145/3335772.3335936}.

\bibitem[Lian et~al.(2015)Lian, Huang, Li, and Liu]{lian2015asynchronous}
Xiangru Lian, Yijun Huang, Yuncheng Li, and Ji~Liu.
\newblock Asynchronous parallel stochastic gradient for nonconvex optimization.
\newblock In \emph{Advances in Neural Information Processing Systems}, 2015.

\bibitem[Malinovsky et~al.(2024)Malinovsky, Richt{\'{a}}rik, Horv{\'{a}}th, and
  Gorbunov]{malinovsky_byzantine_2024}
Grigory Malinovsky, Peter Richt{\'{a}}rik, Samuel Horv{\'{a}}th, and Eduard
  Gorbunov.
\newblock Byzantine robustness and partial participation can be achieved at
  once: Just clip gradient differences.
\newblock In \emph{Advances in Neural Information Processing Systems}, 2024.

\bibitem[Mania et~al.(2017)Mania, Pan, Papailiopoulos, Recht, Ramchandran, and
  Jordan]{mania2017perturbed}
Horia Mania, Xinghao Pan, Dimitris Papailiopoulos, Benjamin Recht, Kannan
  Ramchandran, and Michael~I Jordan.
\newblock Perturbed iterate analysis for asynchronous stochastic optimization.
\newblock In \emph{SIAM Journal on Optimization}, 2017.

\bibitem[Maranjyan et~al.(2025)Maranjyan, Tyurin, and
  Richt{\'a}rik]{ringmaster25}
Arto Maranjyan, Alexander Tyurin, and Peter Richt{\'a}rik.
\newblock Ringmaster {ASGD}: The first asynchronous {SGD} with optimal time
  complexity.
\newblock In \emph{Forty-second International Conference on Machine Learning},
  2025.

\bibitem[Mhamdi et~al.(2018)Mhamdi, Guerraoui, and
  Rouault]{guerraoui2018hidden}
El~Mahdi~El Mhamdi, Rachid Guerraoui, and S{\'{e}}bastien Rouault.
\newblock The hidden vulnerability of distributed learning in byzantium.
\newblock In \emph{International Conference on Machine Learning}, 2018.

\bibitem[Mhamdi et~al.(2021)Mhamdi, Guerraoui, and Rouault]{el2021distributed}
El~Mahdi~El Mhamdi, Rachid Guerraoui, and S{\'{e}}bastien Rouault.
\newblock Distributed momentum for byzantine-resilient stochastic gradient
  descent.
\newblock In \emph{International Conference on Learning Representations}, 2021.

\bibitem[Mishchenko et~al.(2022)Mishchenko, Bach, Even, and
  Woodworth]{mishchenko_asynchronous_2023}
Konstantin Mishchenko, Francis~R. Bach, Mathieu Even, and Blake~E. Woodworth.
\newblock Asynchronous {SGD} beats minibatch {SGD} under arbitrary delays.
\newblock In \emph{Advances in Neural Information Processing Systems}, 2022.

\bibitem[Nesterov(2018)]{nesterov2018lectures}
Yurii Nesterov.
\newblock \emph{Lectures on convex optimization}.
\newblock Springer, 2018.

\bibitem[Nguyen et~al.(2022)Nguyen, Malik, Zhan, Yousefpour, Rabbat, Malek, and
  Huba]{nguyen2022federated}
John Nguyen, Kshitiz Malik, Hongyuan Zhan, Ashkan Yousefpour, Mike Rabbat, Mani
  Malek, and Dzmitry Huba.
\newblock Federated learning with buffered asynchronous aggregation.
\newblock In \emph{International conference on artificial intelligence and
  statistics}, 2022.

\bibitem[Otsuka et~al.(2026)Otsuka, Takezawa, and Yamada]{otsuka2026delayed}
Kaoru Otsuka, Yuki Takezawa, and Makoto Yamada.
\newblock Delayed momentum aggregation: Communication-efficient
  byzantine-robust federated learning with partial participation.
\newblock In \emph{International Conference on Machine Learning}, 2026.

\bibitem[Pillutla et~al.(2022)Pillutla, Kakade, and Harchaoui]{pillutla2019rfa}
Krishna Pillutla, Sham~M. Kakade, and Za{\"{\i}}d Harchaoui.
\newblock Robust aggregation for federated learning.
\newblock In \emph{IEEE Transactions on Signal Processing}, 2022.

\bibitem[Rammal et~al.(2024)Rammal, Gruntkowska, Fedin, Gorbunov, and
  Richt{\'{a}}rik]{rammal_communication_2024}
Ahmad Rammal, Kaja Gruntkowska, Nikita Fedin, Eduard Gorbunov, and Peter
  Richt{\'{a}}rik.
\newblock Communication compression for byzantine robust learning: {N}ew
  efficient algorithms and improved rates.
\newblock In \emph{International Conference on Artificial Intelligence and
  Statistics}, 2024.

\bibitem[Recht et~al.(2011)Recht, Re, Wright, and Niu]{recht2011hogwild}
Benjamin Recht, Christopher Re, Stephen Wright, and Feng Niu.
\newblock Hogwild!: {A} lock-free approach to parallelizing stochastic gradient
  descent.
\newblock In \emph{Advances in neural information processing systems}, 2011.

\bibitem[Ryabinin et~al.(2021)Ryabinin, Gorbunov, Plokhotnyuk, and
  Pekhimenko]{ryabinin2021moshpit}
Max Ryabinin, Eduard Gorbunov, Vsevolod Plokhotnyuk, and Gennady Pekhimenko.
\newblock Moshpit {SGD}: {C}ommunication-efficient decentralized training on
  heterogeneous unreliable devices.
\newblock In \emph{Advances in Neural Information Processing Systems}, 2021.

\bibitem[Shi et~al.(2024)Shi, Yang, and Li]{shi_ordered_2025}
Chang{-}Wei Shi, Yi{-}Rui Yang, and Wu{-}Jun Li.
\newblock Ordered momentum for asynchronous {SGD}.
\newblock In \emph{Advances in Neural Information Processing Systems}, 2024.

\bibitem[Shi et~al.(2025)Shi, Peng, Yuan, Wang, and Ling]{shi2025optimal}
Qiankun Shi, Jie Peng, Kun Yuan, Xiao Wang, and Qing Ling.
\newblock Optimal complexity in byzantine-robust distributed stochastic
  optimization with data heterogeneity.
\newblock In \emph{Journal of Machine Learning Research}, 2025.

\bibitem[Stich \& Karimireddy(2020)Stich and Karimireddy]{stich2020error}
Sebastian~U Stich and Sai~Praneeth Karimireddy.
\newblock The error-feedback framework: Sgd with delayed gradients.
\newblock In \emph{Journal of Machine Learning Research}, 2020.

\bibitem[Tsitsiklis et~al.(1986)Tsitsiklis, Bertsekas, and
  Athans]{tsitsiklis1986distributed}
J.~Tsitsiklis, D.~Bertsekas, and M.~Athans.
\newblock Distributed asynchronous deterministic and stochastic gradient
  optimization algorithms.
\newblock In \emph{IEEE Transactions on Automatic Control}, 1986.

\bibitem[Tyurin \& Richt{\'a}rik(2023)Tyurin and Richt{\'a}rik]{rennala23}
Alexander Tyurin and Peter Richt{\'a}rik.
\newblock Optimal time complexities of parallel stochastic optimization methods
  under a fixed computation model.
\newblock In \emph{Thirty-seventh Conference on Neural Information Processing
  Systems}, 2023.

\bibitem[Tyurin et~al.(2024{\natexlab{a}})Tyurin, Gruntkowska, and
  Richt{\'a}rik]{freya24}
Alexander Tyurin, Kaja Gruntkowska, and Peter Richt{\'a}rik.
\newblock Freya {PAGE}: First optimal time complexity for large-scale nonconvex
  finite-sum optimization with heterogeneous asynchronous computations.
\newblock In \emph{The Thirty-eighth Annual Conference on Neural Information
  Processing Systems}, 2024{\natexlab{a}}.

\bibitem[Tyurin et~al.(2024{\natexlab{b}})Tyurin, Pozzi, Ilin, and
  Richt{\'a}rik]{shadowheart24}
Alexander Tyurin, Marta Pozzi, Ivan Ilin, and Peter Richt{\'a}rik.
\newblock Shadowheart {SGD}: Distributed asynchronous {SGD} with optimal time
  complexity under arbitrary computation and communication heterogeneity.
\newblock In \emph{The Thirty-eighth Annual Conference on Neural Information
  Processing Systems}, 2024{\natexlab{b}}.

\bibitem[Verbraeken et~al.(2020)Verbraeken, Wolting, Katzy, Kloppenburg,
  Verbelen, and Rellermeyer]{verbraeken2020survey}
Joost Verbraeken, Matthijs Wolting, Jonathan Katzy, Jeroen Kloppenburg, Tim
  Verbelen, and Jan~S Rellermeyer.
\newblock A survey on distributed machine learning.
\newblock \emph{Acm computing surveys (csur)}, 53\penalty0 (2):\penalty0 1--33,
  2020.

\bibitem[Xie et~al.(2019)Xie, Koyejo, and Gupta]{xie2019icml}
Cong Xie, Sanmi Koyejo, and Indranil Gupta.
\newblock Zeno: {D}istributed stochastic gradient descent with suspicion-based
  fault-tolerance.
\newblock In \emph{International Conference on Machine Learning}, 2019.

\bibitem[Xie et~al.(2020{\natexlab{a}})Xie, Koyejo, and Gupta]{xie_fall_2019}
Cong Xie, Oluwasanmi Koyejo, and Indranil Gupta.
\newblock Fall of empires: Breaking byzantine-tolerant {SGD} by inner product
  manipulation.
\newblock In \emph{Uncertainty in Artificial Intelligence}, 2020{\natexlab{a}}.

\bibitem[Xie et~al.(2020{\natexlab{b}})Xie, Koyejo, and Gupta]{xie2020zenopp}
Cong Xie, Sanmi Koyejo, and Indranil Gupta.
\newblock Zeno++: {R}obust fully asynchronous {SGD}.
\newblock In \emph{International Conference on Machine Learning},
  2020{\natexlab{b}}.

\bibitem[Yang et~al.(2019)Yang, Yi, Wu, Yuan, Wu, Meng, Hong, Wang, Lin, and
  Johansson]{yang2019survey}
Tao Yang, Xinlei Yi, Junfeng Wu, Ye~Yuan, Di~Wu, Ziyang Meng, Yiguang Hong,
  Hong Wang, Zongli Lin, and Karl~H Johansson.
\newblock A survey of distributed optimization.
\newblock \emph{Annual Reviews in Control}, 47:\penalty0 278--305, 2019.

\bibitem[Yang \& Li(2021)Yang and Li]{yang2021basgd}
Yi-Rui Yang and Wu-Jun Li.
\newblock Basgd: {B}uffered asynchronous sgd for byzantine learning.
\newblock In \emph{International conference on machine learning}, 2021.

\bibitem[Yang \& Li(2023)Yang and Li]{yang2023buffered}
Yi{-}Rui Yang and Wu{-}Jun Li.
\newblock Buffered asynchronous {SGD} for byzantine learning.
\newblock In \emph{Journal of Machine Learning Research}, 2023.

\bibitem[Yang et~al.(2024)Yang, Shi, and Li]{yang2024on}
Yi-Rui Yang, Chang-Wei Shi, and Wu-Jun Li.
\newblock On the effect of batch size in byzantine-robust distributed learning.
\newblock In \emph{International Conference on Learning Representations}, 2024.

\bibitem[Yin et~al.(2018)Yin, Chen, Kannan, and Bartlett]{yin2018byzantine}
Dong Yin, Yudong Chen, Ramchandran Kannan, and Peter Bartlett.
\newblock Byzantine-robust distributed learning: {T}owards optimal statistical
  rates.
\newblock In \emph{International conference on machine learning}, 2018.

\bibitem[Zhang et~al.(2020{\natexlab{a}})Zhang, Jin, Fang, and
  Wang]{zhang2020improved}
Bohang Zhang, Jikai Jin, Cong Fang, and Liwei Wang.
\newblock Improved analysis of clipping algorithms for non-convex optimization.
\newblock \emph{Advances in Neural Information Processing Systems},
  2020{\natexlab{a}}.

\bibitem[Zhang et~al.(2020{\natexlab{b}})Zhang, He, Sra, and
  Jadbabaie]{Zhang2020Why}
Jingzhao Zhang, Tianxing He, Suvrit Sra, and Ali Jadbabaie.
\newblock Why gradient clipping accelerates training: A theoretical
  justification for adaptivity.
\newblock In \emph{International Conference on Learning Representations},
  2020{\natexlab{b}}.

\bibitem[Zhu et~al.(2023)Zhu, Wang, Pang, Wang, Jiao, Song, and
  Jordan]{zhu2023aistats}
Banghua Zhu, Lun Wang, Qi~Pang, Shuai Wang, Jiantao Jiao, Dawn Song, and
  Michael~I. Jordan.
\newblock Byzantine-robust federated learning with optimal statistical rates.
\newblock In \emph{International Conference on Artificial Intelligence and
  Statistics}, 2023.

\end{thebibliography}
\bibliographystyle{iclr2027_conference}

\appendix
\newpage

\section{Notations}
\label{app:notations}

In Algorithm~\ref{alg:stoch-async-hard-sgd}, we write

\begin{itemize}
    \item $\bm x^{(0)} \in \mathbb{R}^d$ the initial parameter,
    \item $n$ the total number of clients,
    \item $q\in [1,\infty]$ the scaling factor,
    \item $T$ the number of iterations.
    \item $S_r$ the start of the $r^\text{th}$ round,
    \item $r(t):= \max\{S_r : S_r \le t\}$ the current round number,
    \item $j(t)$ the current client,
    \item $c_i(t) := \sum_{s=r(t)}^{t-1} \mathbf{1}_{\{j_s=i\}}$ the number of gradients that the client $i$ has already delivered in the current round,
    \item $\bm v^{(t-\tau_t)}_{j_t} = \begin{cases}
            \nabla F(\bm x^{(t-\tau_t)};\xi_{j_t}^{(t-\tau_t)}) &\text{if } j_t \in \gG \\
            * &\text{if } j_t \in \gB
        \end{cases}\quad$ the feedback received,
    \item $\eta_{t}
        \coloneqq \begin{cases}
            \frac{\eta}{n} &\text{if $c_{j_t}(t) = 0$}, \\
            \frac{\eta}{q^{c_{j_t}(t)}}  &\text{else}.
        \end{cases}\quad$ the stepsize,
    \item  $\mathbf{S}_{r(t)}$ the current side information.
\end{itemize}

\section{Analyzing Asynchronous SGD through the Virtual Iterate}
\label{app:virtual_iterate}

\citet{koloskova_sharper_2022, mishchenko_asynchronous_2023} study a virtual iterate $\tilde{\bm x}^{(t)}$, where every update is performed on the current model instead of the outdated version,
\begin{equation} \label{eq:virtual_iterate}
    \tilde{\bm x}^{(t)} := \tilde{\bm x}^{(t-1)} - \eta \nabla F(\bm x^{(t-1)}). %
\end{equation}

The sum of the updates of the virtual iterate over $T$ steps is simply $\tilde S=\tilde{\bm x}^{(0)}+\sum_{t=1}^T\eta \nabla F(\bm x^{(t-1)})$. The key insight is that the sum of the updates of the actual iterate $S=\bm x^{(0)}+\sum_{t=1}^T\eta \nabla F(\bm x^{(t-\tau_t)})$ corresponds almost exactly to $\tilde S$. It only differs by the first factor $\bm x^{(0)}$ and the fact that the update $\nabla F(\bm x^{(0)})$ appears $n$ times in $S$ but only once in $\tilde S$, which comprises later updates instead ($S$ can be seen as being ``delayed'' compared to $\tilde S$). Therefore, both iterates differ by at most $n-1$ stochastic gradients, regardless of $T$ \citep[Lemma~1]{mishchenko_asynchronous_2023}. Standard Synchronous SGD techniques can be used to analyse the convergence of the virtual iterate, and hence that of the real iterate $ \bm x^{(t)} $.

\paragraph{Adapting to the Byzantine Setting.} However, this virtual iterate loses its usefulness in the Byzantine setting. Indeed, it can consist of an arbitrarily large proportion of Byzantine updates, and hence be arbitrarily bad. This contrasts with standard Byzantine distributed learning settings and explains the added assumptions of previous work~\citep{dahan_weight_2024}. We remove this assumption by considering a modified version of the virtual iterate, which we introduce in Appendix~\ref{app:conv_proofs}.

\section{Convergence Proofs}~\label{app:conv_proofs}
Our proofs build on \textit{perturbed iterate analysis}~\citep{mania2017perturbed, stich2020error}.
Existing analyses of asynchronous SGD~\citep{mishchenko_asynchronous_2023, koloskova_sharper_2022} also rely on this technique, approximating the asynchronous iterate by a \textit{synchronized} virtual iterate.
A related construction is the \textit{virtual sequence with restart}~\citep{koloskova2023gradient}, later adapted to incremental gradient methods~\citep{koloskova2024on}.
\citet{islamov24asgrad} combined the two ideas through a two-stage virtual sequence to analyze asynchronous SGD with shuffling.

Inspired by \citet{islamov24asgrad}, our analysis of {\myalgo} also uses a two-stage virtual sequence.
The first stage is the standard approximation of the asynchronous iterate by a synchronized one (Fig.~\ref{fig:async-gap}).
The second approximates this synchronized iterate, which still contains Byzantine updates, by a synchronized iterate free of them (Fig.~\ref{fig:virtual-restart}):
\[
\begin{tikzcd}[column sep=large]
\bm x^{(t)} \arrow[r, "\text{synchronization}"] & \bm y^{(t)} \arrow[r, "\text{de-Byzantinization}"] & \bm z^{(t)} \\[-20pt]
\text{real iterate} & \text{``synchronized'' iterate} & \text{de-Byzantinized iterate}
\end{tikzcd}
\]
Both the de-Byzantinization step and its combination with asynchronous virtual sequences with restart are new to our analysis.

In Section~\ref{app:de-byz}, we show that perturbed iterate analysis also applies to standard Byzantine-robust SGD analysis (with momentum, for generality).
Section~\ref{app:stochastic_hard_reset} then presents the full analysis of {\myalgo} using the two-stage virtual sequence above.

\subsection{Virtual Iterate Analysis for Synchronous Byzantine-robust SGD} \label{app:de-byz}
In this section, we present a new analysis technique for Byzantine-robust SGD.
It recovers the known convergence rate and, more importantly, provides the proof strategy for our main theorem.
Consider the actual iterates of Byzantine-robust SGD with momentum.
We include momentum just for generality ($\alpha = 1$ reduces to SGD).
\[
  \bm x^{(t)} \coloneqq \bm x^{(t-1)} - \eta \Agg\bigl(\{\bm m_i^{(t)}\}_{i=1}^n\bigr), \qquad
  \bm m_i^{(t)} \coloneqq (1-\alpha)\,\bm m_i^{(t-1)} + \alpha \nabla F(\bm x^{(t-1)}; \xi_i^{(t)}).
\]
We construct the virtual sequence $\{\bm y^{(t)}\}_{t=0}^{T-1}$ by
\[
  \bm y^{(t)} \coloneqq \bm x^{(t-1)} - \eta \bar{\bm m}^{(t)}, \qquad \bm y^{(0)} \coloneqq \bm x^{(0)},
\]
that is, the point reached from $\bm x^{(t-1)}$ when the aggregate is replaced by the average momentum of honest clients.

\begin{lemma}[Descent Lemma for Virtual Iterate Analysis for Synchronous Byzantine-robust SGDm]
Suppose the stepsizes satisfy
\(0<\eta_t\le 1/(8L)\).
and denote
\(\mathbb{E}_t[\cdot]
:=\mathbb{E}[\cdot\mid\mathcal{F}_{t-1}]\),
where \(\mathcal{F}_{t-1}\) contains the full history through
round \(t-1\).
Then, for every \(t\ge2\),
\begin{align*}
    \mathbb{E}_t\bigl[f(\bm y^{(t)})\bigr]
    \leq &
    f(\bm y^{(t-1)})
    -\frac{3\eta_t}{8}
    \bigl\|\nabla f(\bm x^{(t-1)})\bigr\|^2
    +
    \frac{5\eta_t}{4}
    \mathbb{E}_t
    \bigl\|
        \bar{\bm m}^{(t)}
        -\nabla f(\bm x^{(t-1)})
    \bigr\|^2
    +
    \underbrace{\frac{3}{2\eta_t}
    \bigl\|
        \bm x^{(t-1)}-\bm y^{(t-1)}
    \bigr\|^2}_{=\frac{3\eta_{t-1}^{2}}{2\eta_t}
    \bigl\|
        \bm m^{(t-1)}-\bar{\bm m}^{(t-1)}
    \bigr\|^2}.
\end{align*}
\end{lemma}
\begin{proof}
By the definition of the virtual iterate,
\[
    \bm y^{(t)}
    =
    \bm y^{(t-1)}
    -\eta_t\bar{\bm m}^{(t)}
    +\bigl(\bm x^{(t-1)}-\bm y^{(t-1)}\bigr).
\]
Therefore, \(L\)-smoothness gives
\begin{align*}
    f(\bm y^{(t)})
    \le{}&
    f(\bm y^{(t-1)})
    -\eta_t
    \left\langle
        \nabla f(\bm y^{(t-1)}),
        \bar{\bm m}^{(t)}
    \right\rangle
    \\
    &+
    \left\langle
        \nabla f(\bm y^{(t-1)}),
        \bm x^{(t-1)}-\bm y^{(t-1)}
    \right\rangle
    \\
    &+
    \frac{L\eta_t^2}{2}
    \left\|
        \bar{\bm m}^{(t)}
        -\frac{\bm x^{(t-1)}-\bm y^{(t-1)}}{\eta_t}
    \right\|^2.
\end{align*}
Adding and subtracting \(\nabla f(\bm x^{(t-1)})\)
inside the inner product, while retaining the
virtual iterate gap, yields
\begin{align*}
    f(\bm y^{(t)})
    \le{}&
    f(\bm y^{(t-1)})
    -\eta_t
    \left\langle
        \nabla f(\bm y^{(t-1)}),
        \nabla f(\bm x^{(t-1)})
    \right\rangle
    \\
    &-
    \eta_t
    \left\langle
        \nabla f(\bm y^{(t-1)}),
        \bar{\bm m}^{(t)}
        -\nabla f(\bm x^{(t-1)})
        -\frac{\bm x^{(t-1)}-\bm y^{(t-1)}}{\eta_t}
    \right\rangle
    +
    \frac{L\eta_t^2}{2}
    \left\|
        \bar{\bm m}^{(t)}
        -\frac{\bm x^{(t-1)}-\bm y^{(t-1)}}{\eta_t}
    \right\|^2.
\end{align*}

For the first inner product, polarization identity and
\(L\)-smoothness imply
\begin{align*}
    -\eta_t
    \left\langle
        \nabla f(\bm y^{(t-1)}),
        \nabla f(\bm x^{(t-1)})
    \right\rangle
    &=
    -\frac{\eta_t}{2}
    \bigl\|\nabla f(\bm x^{(t-1)})\bigr\|^2
    -\frac{\eta_t}{2}
    \bigl\|\nabla f(\bm y^{(t-1)})\bigr\|^2
    +
    \frac{\eta_t}{2}
    \bigl\|
        \nabla f(\bm x^{(t-1)})
        -\nabla f(\bm y^{(t-1)})
    \bigr\|^2
    \\
    &\le
    -\frac{\eta_t}{2}
    \bigl\|\nabla f(\bm x^{(t-1)})\bigr\|^2
    -\frac{\eta_t}{2}
    \bigl\|\nabla f(\bm y^{(t-1)})\bigr\|^2
    +
    \frac{\eta_tL^2}{2}
    \bigl\|\bm x^{(t-1)}-\bm y^{(t-1)}\bigr\|^2.
\end{align*}
For the second inner product, Young's inequality gives
\begin{align*}
    &-\eta_t
    \left\langle
        \nabla f(\bm y^{(t-1)}),
        \bar{\bm m}^{(t)}
        -\nabla f(\bm x^{(t-1)})
        -\frac{\bm x^{(t-1)}-\bm y^{(t-1)}}{\eta_t}
    \right\rangle
    \\
    &\quad\le
    \frac{\eta_t}{2}
    \bigl\|\nabla f(\bm y^{(t-1)})\bigr\|^2
    +
    \frac{\eta_t}{2}
    \left\|
        \bar{\bm m}^{(t)}
        -\nabla f(\bm x^{(t-1)})
        -\frac{\bm x^{(t-1)}-\bm y^{(t-1)}}{\eta_t}
    \right\|^2.
\end{align*}
Combining these inequalities cancels the squared gradient
norm at \(\bm y^{(t-1)}\), leaving
\begin{align*}
    f(\bm y^{(t)})
    \le &
    f(\bm y^{(t-1)})
    -\frac{\eta_t}{2}
    \bigl\|\nabla f(\bm x^{(t-1)})\bigr\|^2
    +
    \frac{\eta_tL^2}{2}
    \bigl\|\bm x^{(t-1)}-\bm y^{(t-1)}\bigr\|^2
    \\
    &+
    \frac{\eta_t}{2}
    \left\|
        \bar{\bm m}^{(t)}
        -\nabla f(\bm x^{(t-1)})
        -\frac{\bm x^{(t-1)}-\bm y^{(t-1)}}{\eta_t}
    \right\|^2
    +
    \frac{L\eta_t^2}{2}
    \left\|
        \bar{\bm m}^{(t)}
        -\frac{\bm x^{(t-1)}-\bm y^{(t-1)}}{\eta_t}
    \right\|^2.
\end{align*}

Next, inserting \(\nabla f(\bm x^{(t-1)})\) into the
last squared norm gives
\begin{align*}
    \left\|
        \bar{\bm m}^{(t)}
        -\frac{\bm x^{(t-1)}-\bm y^{(t-1)}}{\eta_t}
    \right\|^2
    &\le
    2\bigl\|\nabla f(\bm x^{(t-1)})\bigr\|^2
    +
    2\left\|
        \bar{\bm m}^{(t)}
        -\nabla f(\bm x^{(t-1)})
        -\frac{\bm x^{(t-1)}-\bm y^{(t-1)}}{\eta_t}
    \right\|^2.
\end{align*}
Moreover,
\begin{align*}
    \left\|
        \bar{\bm m}^{(t)}
        -\nabla f(\bm x^{(t-1)})
        -\frac{\bm x^{(t-1)}-\bm y^{(t-1)}}{\eta_t}
    \right\|^2
    &\le
    2\bigl\|
        \bar{\bm m}^{(t)}
        -\nabla f(\bm x^{(t-1)})
    \bigr\|^2
    +
    \frac{2}{\eta_t^2}
    \bigl\|\bm x^{(t-1)}-\bm y^{(t-1)}\bigr\|^2.
\end{align*}
Substituting both bounds and collecting terms, we obtain
\begin{align*}
    f(\bm y^{(t)})
    \le &
    f(\bm y^{(t-1)})
    -\eta_t\left(\frac12-L\eta_t\right)
    \bigl\|\nabla f(\bm x^{(t-1)})\bigr\|^2
    \\
    &+
    \eta_t(1+2L\eta_t)
    \bigl\|
        \bar{\bm m}^{(t)}
        -\nabla f(\bm x^{(t-1)})
    \bigr\|^2
    \\
    &+
    \left(
        \frac1{\eta_t}+2L+\frac{\eta_tL^2}{2}
    \right)
    \bigl\|\bm x^{(t-1)}-\bm y^{(t-1)}\bigr\|^2.
\end{align*}
Since \(L\eta_t\le1/8\),
\[
    \frac12-L\eta_t\ge\frac38,
    \qquad
    1+2L\eta_t\le\frac54,
    \qquad
    1+2L\eta_t+\frac{L^2\eta_t^2}{2}
    \le\frac{161}{128}\le\frac32.
\]
Consequently,
\begin{align*}
    f(\bm y^{(t)})
    \le &
    f(\bm y^{(t-1)})
    -\frac{3\eta_t}{8}
    \bigl\|\nabla f(\bm x^{(t-1)})\bigr\|^2
    +
    \frac{5\eta_t}{4}
    \bigl\|
        \bar{\bm m}^{(t)}
        -\nabla f(\bm x^{(t-1)})
    \bigr\|^2
    +
    \frac{3}{2\eta_t}
    \bigl\|\bm x^{(t-1)}-\bm y^{(t-1)}\bigr\|^2.
\end{align*}
Taking conditional expectations and for \(t\ge2\), the real and virtual updates imply
\begin{align*}
    \bm x^{(t-1)}-\bm y^{(t-1)}
    &=
    \bigl(\bm x^{(t-2)}-\eta_{t-1}\bm m^{(t-1)}\bigr)
    -
    \bigl(\bm x^{(t-2)}
          -\eta_{t-1}\bar{\bm m}^{(t-1)}\bigr)
    \\
    &=
    -\eta_{t-1}
    \bigl(\bm m^{(t-1)}-\bar{\bm m}^{(t-1)}\bigr).
\end{align*}
Substitution proves the Lemma.
\end{proof}

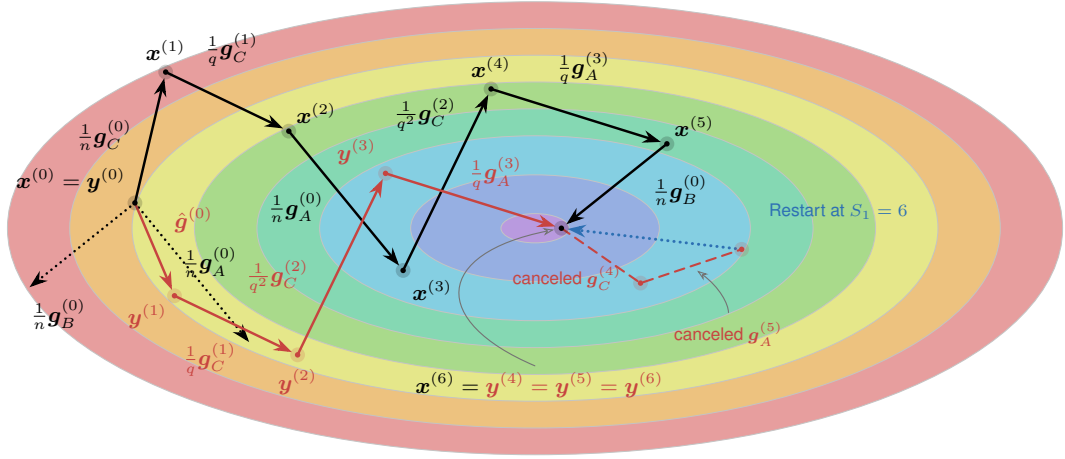
\begin{figure}[t]
\centering
    \begin{minipage}{\linewidth}
    \centering
    \resizebox{\linewidth}{!}{%
    \begin{tikzpicture}[
      x=1.25cm,y=1.20cm,
      font=\small, line cap=round,line join=round,
      >={Stealth[length=2.6mm,width=1.8mm]},
      every node/.style={inner sep=1.5pt},
      actual/.style={draw=black,line width=1pt},
      virtual/.style={draw=virtualred,line width=1pt},
      initial/.style={draw=black,line width=1pt,
        dash pattern=on 0pt off 2.3pt},
      canceled/.style={draw=virtualred,line width=.85pt,
        dash pattern=on 4pt off 2.7pt},
      restart/.style={draw=restartblue,line width=1.2pt,
        dash pattern=on 0pt off 2.4pt},
      callout/.style={draw=black!55,line width=.45pt,
        -{Stealth[open,length=1.6mm,width=1.1mm]},
        shorten <=2pt,shorten >=3pt},
      gradient/.style={font=\small},
      note/.style={font=\sffamily\scriptsize}
    ]
      \definecolor{levelone}{RGB}{231,158,158}
      \definecolor{leveltwo}{RGB}{237,194,127}
      \definecolor{levelthree}{RGB}{229,231,138}
      \definecolor{levelfour}{RGB}{169,218,139}
      \definecolor{levelfive}{RGB}{133,216,179}
      \definecolor{levelsix}{RGB}{133,207,226}
      \definecolor{levelseven}{RGB}{150,175,226}
      \definecolor{leveleight}{RGB}{185,154,223}
      \definecolor{virtualred}{HTML}{C74440}
      \definecolor{restartblue}{HTML}{2D70B3}
    
      \foreach \level/\shade in {
        36/levelone,28/leveltwo,21/levelthree,15/levelfour,
        10/levelfive,6/levelsix,2/levelseven,.15/leveleight}{
        \path[fill=\shade,draw=black!23,line width=.35pt]
          (0,0) ellipse [x radius={sqrt(\level)},y radius={sqrt(\level/5)}];
      }
    
      \coordinate (x0) at (-4.55, .30);
      \coordinate (x1) at (-4.20,1.85);
      \coordinate (x2) at (-2.80,1.15);
      \coordinate (x3) at (-1.50,-.50);
      \coordinate (x4) at (-.50,1.65);
      \coordinate (x5) at ( 1.50,1.00);
      \coordinate (x6) at ( .30,0);
    
      \coordinate (y0) at (x0);
      \coordinate (y1) at ($(x1)+(x3)-(x2)+(x6)-(x5)$); %
      \coordinate (y2) at ($(y1)+(x2)-(x1)$);          %
      \coordinate (y3) at ($(y2)+(x4)-(x3)$);          %
      \coordinate (y4) at ($(y3)+(x5)-(x4)$);          %
      \coordinate (canceledC) at ($(y4)+(.90,-.65)$);
      \coordinate (canceledA) at ($(canceledC)+(1.15,.40)$);
      \coordinate (initialA) at ($(x0)+(x3)-(x2)$);
      \coordinate (initialB) at ($(x0)+(x6)-(x5)$);
    
      \draw[->,initial] (x0) -- (initialA);
      \draw[->,initial] (x0) -- (initialB);
      \draw[canceled] (y4) -- (canceledC) -- (canceledA);
      \draw[->,restart,shorten >=2.5pt] (canceledA) -- (x6);
    
      \foreach \a/\b in {0/1,1/2,2/3,3/4,4/5,5/6}
        \draw[->,actual,shorten >=2.5pt] (x\a) -- (x\b);
      \foreach \a/\b in {0/1,1/2,2/3,3/4}
        \draw[->,virtual,shorten >=2.5pt] (y\a) -- (y\b);
    
      \foreach \i in {0,...,6}{
        \fill[black,opacity=.18] (x\i) circle[radius=3pt];
        \fill[black] (x\i) circle[radius=1.1pt];
      }
      \foreach \point in {y1,y2,y3,canceledC,canceledA}{
        \fill[virtualred,opacity=.18] (\point) circle[radius=3pt];
        \fill[virtualred] (\point) circle[radius=1.1pt];
      }
    
      \node[anchor=south east] at ($(x0)+(-.06,.07)$) {$\bm x^{(0)}=\bm y^{(0)}$};
      \node[above=3pt] at (x1) {$\bm x^{(1)}$};
      \node[above right=2pt] at (x2) {$\bm x^{(2)}$};
      \node[below right=3pt] at (x3) {$\bm x^{(3)}$};
      \node[above=3pt] at (x4) {$\bm x^{(4)}$};
      \node[above right=2pt] at (x5) {$\bm x^{(5)}$};
      \node[text=virtualred,below left=2pt] at (y1) {$\bm y^{(1)}$};
      \node[text=virtualred,below=4pt] at (y2) {$\bm y^{(2)}$};
      \node[text=virtualred,above left=3pt] at (y3) {$\bm y^{(3)}$};
      \node (finalIterateLabel) at (.05,-1.85)
        {$\bm x^{(6)}=\textcolor{virtualred}{\bm y^{(4)}=\bm y^{(5)}=\bm y^{(6)}}$};
    
      \node[gradient] at (-4.90,1.10) {$\frac{1}{n}\bm g_C^{(0)}$};
      \node[gradient] at (-3.45,2.12) {$\frac{1}{q}\bm g_C^{(1)}$};
      \node[gradient] at (-2.72, .25) {$\frac{1}{n}\bm g_A^{(0)}$};
      \node[gradient] at (-1.25,1.35) {$\frac{1}{q^2}\bm g_C^{(2)}$};
      \node[gradient] at ( .55,1.90) {$\frac{1}{q}\bm g_A^{(3)}$};
      \node[gradient] at (1.65, .45) {$\frac{1}{n}\bm g_B^{(0)}$};
    
      \node[gradient] at (-3.70,-.40) {$\frac{1}{n}\bm g_A^{(0)}$};
      \node[gradient] at (-5.42,-1.02) {$\frac{1}{n}\bm g_B^{(0)}$};
    
      \begin{scope}[every node/.append style={gradient,text=virtualred}]
        \node[font=\footnotesize] at (-3.88,.10)
          {$\hat{\bm g}^{(0)}$};
        \node at (-3.70,-1.55) {$\frac{1}{q}\bm g_C^{(1)}$};
        \node at (-2.93,-.57) {$\frac{1}{q^2}\bm g_C^{(2)}$};
        \node at (-.45,.65) {$\frac{1}{q}\bm g_A^{(3)}$};
      \end{scope}
    
      \node[note,text=virtualred,anchor=east] at (1.03,-.59)
        {canceled $\bm g_C^{(4)}$};
      \node[note,text=virtualred] (canceledALabel) at (2.20,-1.25)
        {canceled $\bm g_A^{(5)}$};
      \node[note,text=restartblue] at (3.45,.20) {Restart at $S_1=6$};

      \draw[callout] (finalIterateLabel.north)
        .. controls (-1.20,-1.10) and (-1.15,-.30) .. (x6);
      \draw[callout] (canceledALabel.north)
        .. controls (2.20,-.90) and (2.10,-.70)
        .. ($(canceledC)!.5!(canceledA)$);
    
    \end{tikzpicture}%
    }
    \caption{Actual and virtual iterates with a hard restart.}
    \label{fig:async-gap}
    \end{minipage}

    \par\vspace{\baselineskip}
    \noindent\hfill
\end{figure}

\begin{figure}
    \centering
    \begin{tikzpicture}[
      x={(0cm,-.55cm)},y={(1.10cm,0cm)},
      font=\normalsize,
      flowarrow/.style={line width=.9pt,-{Latex[length=1.8mm,width=1.2mm]},
        line cap=round,line join=round},
      redarrow/.style={flowarrow,draw=yred!85!black},
      bluearrow/.style={flowarrow,draw=zblue!85!black},
      ylab/.style={text=yred!85!black},
      zlab/.style={text=zblue!85!black},
      vectorlabel/.style={font=\normalsize,fill=white,inner sep=2pt},
      reset/.style={draw=resetgray,line width=.6pt,densely dashed,
        -{Latex[length=1.5mm,width=1mm]}},
    ]
    
    \def\alphaOne{0.62}
    \def\alphaTwo{0.58}
    \def\alphaThree{0.68}
    
    \coordinate (O)  at (0,0);
    \coordinate (Y1) at (2.9,1.8);
    \coordinate (Z1) at (-2.3,1.8);
    
    \coordinate (V1) at (-2.0,1.5);
    \coordinate (V2) at (3.2,1.6);
    \coordinate (V3) at (-2.0,1.9);
    \coordinate (Z2) at ($(Z1)+(V1)$);
    \coordinate (Z3) at ($(Z2)+(2.2,1.7)$);
    \coordinate (Z4) at ($(Z3)+(V3)$);
    \coordinate (Y2) at ($(Y1)+\alphaOne*(V1)$);
    \coordinate (Y3) at ($(Y2)+\alphaTwo*(V2)$);
    \coordinate (Y4) at ($(Y3)+\alphaThree*(V3)$);
    
    \coordinate (Y5) at (Y4);
    \coordinate (Y6) at (Y4);
    \coordinate (Z5) at (Z4);
    \coordinate (Z6) at (Z4);
    
    \coordinate (Y7) at ($(Y6)+(2.3,2.8)$);
    \coordinate (Z7) at ($(Y6)+(-1.8,2.55)$);
    
    \draw[reset] (Z6) -- (Y6)
      node[pos=.45,right=9pt,vectorlabel,text=resetgray]
        {reset at $S_1=6$};
    
    \draw[redarrow] (O) -- (Y1)
      node[pos=.46,below left=5pt,vectorlabel,ylab] {$\hat{\bm g}^{(0)}$};
    \draw[bluearrow] (O) -- (Z1)
      node[pos=.49,above left=4pt,vectorlabel,zlab] {$\bar{\bm g}^{(0)}$};
    
    \draw[redarrow] (Y1) -- (Y2);
    \draw[redarrow] (Y2) -- (Y3);
    \draw[redarrow] (Y3) -- (Y4);
    \draw[bluearrow] (Z1) -- (Z2);
    \draw[bluearrow] (Z2) -- (Z3);
    \draw[bluearrow] (Z3) -- (Z4);
    
    \draw[redarrow] (Y6) -- (Y7)
      node[pos=.58,below=5pt,vectorlabel,ylab] {$\hat{\bm g}^{(6)}$};
    \draw[bluearrow] (Y6) -- (Z7)
      node[pos=.62,above=5pt,vectorlabel,zlab] {$\bar{\bm g}^{(6)}$};
    
    \fill[black!80] (O) circle (1.6pt);
    \node[left=8pt] at (O) {$\bm x^{(0)}=\bm y^{(0)}=\bm z^{(0)}$};
    \foreach \k in {1,2,3,4,7}{
      \fill[yred!85!black] (Y\k) circle (1.6pt);
      \fill[zblue!85!black] (Z\k) circle (1.6pt);
    }
    \node[ylab,below=6pt] at (Y1) {$\bm y^{(1)}$};
    \node[ylab,above=5pt] at (Y2) {$\bm y^{(2)}$};
    \node[ylab,below=6pt] at (Y3) {$\bm y^{(3)}$};
    \node[zlab,above=6pt] at (Z1) {$\bm z^{(1)}$};
    \node[zlab,above=6pt] at (Z2) {$\bm z^{(2)}$};
    \node[zlab,below=6pt] at (Z3) {$\bm z^{(3)}$};
    \node[ylab,anchor=north,inner sep=3pt] (Yidentity)
      at ($(Y6)+(3.3,0)$)
      {$\bm y^{(4)}=\bm y^{(5)}=\bm y^{(6)}\textcolor{black}{=\bm x^{(6)}}$};
    \draw[yred!45!black,line width=.35pt] (Y6) -- (Yidentity.north);
    \node[zlab,above=9pt] at (Z6) {$\bm z^{(4)}=\bm z^{(5)}=\bm z^{(6)}$};
    \node[ylab,right=5pt] at (Y7) {$\bm y^{(7)}$};
    \node[zlab,right=5pt] at (Z7) {$\bm z^{(7)}$};
    
    \end{tikzpicture}%
        \caption{Virtual iterates with clipping and a hard restart.}
        \label{fig:virtual-restart}
\end{figure}
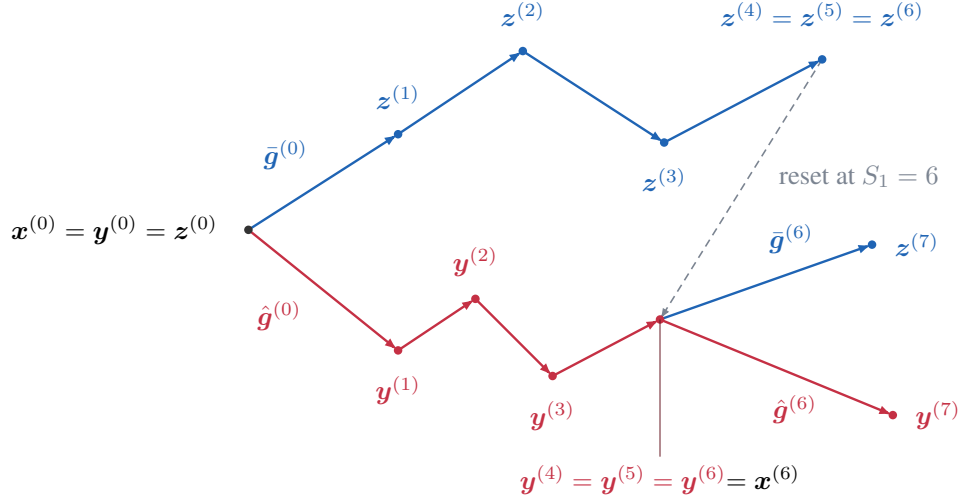

Combining this lemma with \citet[Lemmas~9 and~11]{karimireddy2021icml} yields the same convergence rate up to a constant factor.
This construction also provides key intuition for {\myalgo} and its analysis. 
It corresponds to the special case of {\myalgo} in which every round is a synchronization round ($q \to \infty$).

\subsection{Proof of the Main Theorem}~\label{app:stochastic_hard_reset}

Let us restate the main theorem in its precise form.
\begin{mainrestated}[Non-convex Convergence rate]
Let $T\geq2$ and $q>1$.  Suppose that Assumptions
\ref{ass:objective},~\ref{ass:stochastic-gradients}, and
\ref{ass:bdd_gradient}, as well as
Assumption~\ref{ass:stoch-bounded-local-contributions} with $G' \leq 3G$ hold, together with
the local $(c,\delta)$-robust aggregator $\mathbb{A}$, and use $\lambda_t=G$.  
Define initial distance $\Delta$, effective time steps $T_{\mathrm{eff}}$, number of restart rounds $N_T$, and restart ratio $\theta_T$ as
{\small
\begin{equation*}
    \Delta
    \coloneqq
    f(\bm x^{(0)})-f_*, \quad
    T_{\mathrm{eff}}
    \eqqcolon
    1 +\sum_{t=1}^{T-1}\hat\eta_t/\eta, \quad 
    N_T
    \coloneqq
    1+\sum_{t=1}^{T-1}\indi_{\{t=r(t)\}}, \quad
    \theta_T
    \coloneqq
    \frac{N_T}{T_{\mathrm{eff}}}.
\end{equation*}
}
Suppose that $\Delta>0$, and choose the output iterate according to
{\small
\begin{equation} \label{eq:weighted_average_in_main_thm}
    \mathbb P(\bm x^{\mathrm{out}}=\bm x^{(0)})
    =
    \frac{1}{T_{\mathrm{eff}}}, \quad
    \mathbb P(\bm x^{\mathrm{out}}=\bm x^{(t)})
    =
    \frac{\hat\eta_t}{\eta T_{\mathrm{eff}}},
    \qquad t\in\{1,\ldots,T-1\}.
\end{equation}
}
Choose the stepsize $\eta > 0$ as
{\small
\begin{align*}
    \eta
    =
    \min\Bigg\{
        &\sqrt{\frac{\Delta}{L\sigma^2(1-\theta_T+\theta_T/|\gG|)T_{\mathrm{eff}}}},
        &\left(
            \frac{\Delta}
            {L^2T_{\mathrm{eff}}(1-\theta_T)
            \left[
                9(n-1)^2G^2
                +4c\delta\sigma^2
                +\dfrac{8G^2|\gB|^2}{(q-1)^2}
            \right]}
        \right)^{1/3},
        \frac{1}{4L}
    \Bigg\}.
\end{align*}
}
We have
\begin{align*}
    \E\|\nabla f(\bm x^{\mathrm{out}})\|^2
    =\gO\Bigg(
        &\sqrt{\frac{L\sigma^2\Delta}{T_{\mathrm{eff}}}
            \left(1-\theta_T+\frac{\theta_T}{|\gG|}\right)}\\
        &+\left(\frac{L\Delta}{T_{\mathrm{eff}}}\right)^{2/3}
        \left[
            (1-\theta_T)\left(
                9(n-1)^2G^2
                +
                4c\delta\sigma^2
                +\frac{8G^2|\gB|^2}{(q-1)^2}
            \right)
        \right]^{1/3}\\
        &+\frac{L\Delta}{T_{\mathrm{eff}}}
        +\left(
            4c\delta\sigma^2
            +\frac{8G^2|\gB|^2}{(q-1)^2}
        \right)
        \theta_T
    \Bigg).
\end{align*}

\end{mainrestated}

Let the real iterate be
$$
    \bm x^{(t)} =
        \bm x^{(t-1)} - \eta_t \bm g^{(t-\tau_t)}_{j_t}, \quad \text{where } t - \tau_t \geq r(t),
$$
with the stepsize $\eta_t$ in \eqref{eq:stepsize}.  First arrivals use the local aggregator with side information from the last synchronized round, and repeated arrivals use clipping at radius $G$.

Define the first virtual iterate (de-asynchronized version) as
$$
    \bm y^{(t+1)} = \begin{cases}
    \bm y^{(t)} - \hat \eta_t \bm g^{(t)}_{j_t}, & \text{if } t > r(t),\\
    \bm x^{(t)} - \frac{\eta}{n} \displaystyle\sum_{i=1}^n \bm g^{(t)}_i , & \text{if } t = r(t),
    \end{cases},
\quad \text{and} \quad
    \bm y^{(1)} = \bm x^{(0)} - \frac{\eta}{n} \displaystyle\sum_{i=1}^n \bm g^{(0)}_i.
$$
where
$$
    \hat\eta_t
    \coloneqq
    \begin{cases}
        \eta, & \text{if }t=r(t), \\
        \eta_{\nexxt(t+1,j_t)}, & \text{if }t>r(t) \text{ and the gradient computation is not canceled by the server} ,\\
        0 & t > r(t) \text{ and the gradient computation is canceled by the server}.
    \end{cases}
$$

Furthermore, we define another virtual iterate to remove the effect of clipping and Byzantine updates:
$$
    \bm z^{(t+1)} = \begin{cases}
    \bm z^{(t)} - \hat \eta_t \nabla F(\bm x^{(t)}; \xi_{j_t}^{(t)}), & \text{if } t > r(t),\\
    \bm y^{(t)} - \eta \cdot \frac{1}{|\gG|} \sum_{i \in \gG} \nabla F(\bm x^{(t)}; \xi_i^{(t)}), & \text{if } t = r(t),
    \end{cases}
$$
and
$$
    \bm z^{(1)} = \bm x^{(0)} - \eta \cdot \frac{1}{|\gG|} \sum_{i \in \gG} \nabla F(\bm x^{(0)}; \xi_i^{(0)}).
$$
If the job is canceled (i.e., $\hat \eta_t = 0$), define $\bm y^{(t+1)} = \bm y^{(t)}$ and $\bm z^{(t+1)} = \bm z^{(t)}$.
By construction, we have at every synchronization time $S_r$, $\bm y^{(S_r)} = \bm x^{(S_r)}.$

The definition of $\bm z^{(1)}$ gives one additional stochastic descent
step before the virtual-iterate recursion starts.
\begin{lemma}[Initial stochastic descent step]
\label{lem:stoch-hard-initial-descent}
If $\eta\leq1/L$, then
\begin{align*}
    \E[f(\bm z^{(1)})\mid\bm x^{(0)}]
    &\leq
    f(\bm x^{(0)})
    -\frac{\eta}{2}\|\nabla f(\bm x^{(0)})\|^2
    +\frac{L\eta^2\sigma^2}{2|\gG|}.
\end{align*}
\end{lemma}

\begin{proof}
Counting the initial round as restart-like, let
\[
    \overline{\bm g}^{(0)}
    \coloneqq
    \frac{1}{|\gG|}\sum_{i\in\gG}
    \nabla F(\bm x^{(0)};\xi_i^{(0)}).
\]
Conditional independence, unbiasedness, and the variance bound give
\begin{align*}
    \E[\overline{\bm g}^{(0)}\mid\bm x^{(0)}]
    &=\nabla f(\bm x^{(0)}),\\
    \E[\|\overline{\bm g}^{(0)}-\nabla f(\bm x^{(0)})\|^2
        \mid\bm x^{(0)}]
    &\leq\frac{\sigma^2}{|\gG|}.
\end{align*}
Therefore, $L$-smoothness and
$\bm z^{(1)}=\bm x^{(0)}-\eta\overline{\bm g}^{(0)}$ imply
\begin{align*}
    \E[f(\bm z^{(1)})\mid\bm x^{(0)}]
    &\leq
    f(\bm x^{(0)})
    -\eta\|\nabla f(\bm x^{(0)})\|^2\\
    &\quad+
    \frac{L\eta^2}{2}
    \left(
        \|\nabla f(\bm x^{(0)})\|^2
        +
        \frac{\sigma^2}{|\gG|}
    \right)\\
    &\leq
    f(\bm x^{(0)})
    -\frac{\eta}{2}\|\nabla f(\bm x^{(0)})\|^2
    +\frac{L\eta^2\sigma^2}{2|\gG|},
\end{align*}
where the last inequality uses $L\eta\leq1$.
\end{proof}

At a synchronization time $t=r(t)$, the aggregate used by the first
virtual iterate satisfies
\[
    \frac{1}{n}\sum_{i=1}^n
    \mathbb A(\bm v^{(t)}_i;S)
    =
    \frac{1}{n}\sum_{i=1}^n\bm g_i^{(t)}.
\]

We use a local $(c,\delta)$-robust aggregator in the
sense of the main text.

\begin{lemma}[Synchronization aggregation error]
\label{lem:stoch-hard-sync-aggregation}
Under Assumption~\ref{ass:stochastic-gradients}, at every synchronization
time $t$,
\begin{align*}
    \E\left[
        \left\|
            \frac{1}{|\gG|}\sum_{i\in\gG}
            \nabla F(\bm x^{(t)};\xi_i^{(t)})
            -\frac{1}{n}\sum_{i=1}^n
            \mathbb A(\bm v^{(t)}_i; \mathbf{S}_{r(t)})
        \right\|^2
        \mid\gF_t
    \right]
    \leq
    2c\delta\sigma^2.
\end{align*}
\end{lemma}

\begin{proof}
For any $i,j\in\gG$, conditional independence and unbiasedness give
\begin{align*}
    \E[\|\nabla F(\bm x^{(t)};\xi_i^{(t)})
        -\nabla F(\bm x^{(t)};\xi_j^{(t)})\|^2\mid\gF_t]
    &=
    \E[\|\nabla F(\bm x^{(t)};\xi_i^{(t)})
        -\nabla f(\bm x^{(t)})\|^2\mid\gF_t]\\
    &\quad+
    \E[\|\nabla F(\bm x^{(t)};\xi_j^{(t)})
        -\nabla f(\bm x^{(t)})\|^2\mid\gF_t]\\
    &\leq2\sigma^2.
\end{align*}
Thus the pairwise radius $\rho^2$ in
Definition~\ref{def:robust_agg} is at most $2\sigma^2$.  Applying the local $(c,\delta)$-robust
aggregator definition~\ref{def:robust_agg} conditionally on $\gF_t$ proves the claim.
\end{proof}

\begin{lemma}[Stochastic virtual-iterate relationship]
\label{lem:stoch-hard-virtual-iterate-relationship}
For every $t$ after the first restart,
\begin{align*}
    \bm y^{(t+1)}-\bm z^{(t+1)}
    &=
    \eta\left[
        \frac{1}{|\gG|}\sum_{i\in\gG}
        \nabla F(\bm x^{(r(t))};\xi_i^{(r(t))})
        -\frac{1}{n}\sum_{i=1}^n
        \mathbb A(\bm v^{(r(t))}_i; \mathbf{S}_{r(t)})
    \right]\\
    &\quad+
    \sum_{k=r(t)+1}^{t}
    \hat\eta_k
    \left[
        \nabla F(\bm x^{(k)};\xi_{j_k}^{(k)})-\bm g_{j_k}^{(k)}
    \right].
\end{align*}
The second sum is zero when $t=r(t)$.  Moreover,
\[
    \E\|\bm y^{(t+1)}-\bm z^{(t+1)}\|^2
    \leq
    \eta^2
    \left(
        4c\delta\sigma^2
        +\frac{8G^2|\gB|^2}{(q-1)^2}
    \right).
\]
\end{lemma}

\begin{proof}
Let $r=r(t)$.  At the synchronization round,
\[
    \bm y^{(r+1)}-\bm z^{(r+1)}
    =
    \eta\left[
        \frac{1}{|\gG|}\sum_{i\in\gG}
        \nabla F(\bm x^{(r)};\xi_i^{(r)})
        -\frac{1}{n}\sum_{i=1}^n
        \mathbb A(\bm v^{(r)}_i; \mathbf{S}_{r(t)})
    \right].
\]
Every subsequent round $k>r$ satisfies
\[
    \bm y^{(k+1)}-\bm z^{(k+1)}
    =
    \bm y^{(k)}-\bm z^{(k)}
    +\hat\eta_k
    \left[
        \nabla F(\bm x^{(k)};\xi_{j_k}^{(k)})-\bm g_{j_k}^{(k)}
    \right].
\]
Unrolling this recursion proves the first display.

For every good client, clipping is inactive by
Assumption~\ref{ass:bdd_gradient}, and hence
$\nabla F(\bm x^{(k)};\xi_{j_k}^{(k)})-\bm g_{j_k}^{(k)}=0$.
The sum below contains only repeated-arrival contributions, whose norms
are at most $G$ by non-restart clipping at radius $G$.  Together with
Assumption~\ref{ass:bdd_gradient}, this gives
\begin{align*}
    \left\|
        \sum_{k=r(t)+1}^{t}
        \frac{\hat\eta_k}{\eta}
        \left[
            \nabla F(\bm x^{(k)};\xi_{j_k}^{(k)})-\bm g_{j_k}^{(k)}
        \right]
    \right\|
    &\leq
    2G\sum_{k=r(t)+1}^{t}
    \frac{\hat\eta_k}{\eta}
    \indi_{\{j_k\in\gB\}}\\
    &\leq
    \frac{2G|\gB|}{q-1}.
\end{align*}
The last inequality follows by summing, for each Byzantine client, the
geometric sequence $q^{-1}+q^{-2}+\cdots$ within the round.  Therefore,
Lemma~\ref{lem:stoch-hard-sync-aggregation} gives
\begin{align*}
    \E\|\bm y^{(t+1)}-\bm z^{(t+1)}\|^2
    &\leq
    \eta^2\left(
    4c\delta\sigma^2
    +\frac{8G^2|\gB|^2}{(q-1)^2}
    \right).
\end{align*}
This proves the result.
\end{proof}

\begin{lemma}[Real--virtual iterate difference after synchronization]
\label{lem:stoch-hard-real-virtual-difference}
Suppose Assumption~\ref{ass:stoch-bounded-local-contributions} holds, and
let $r$ be a synchronization round.  At non-restart times $t>r(t)$,
under the same hard-restart indexing as in the deterministic proof,
\[
    \bm x^{(t)}-\bm y^{(t)}
    =
    \sum_{m\in[n]\setminus\{j_t\}}
    \eta_{\nexxt(t,m)}
    \bm g_m^{(\prev(t,m))}.
\]
Consequently,
\[
    \|\bm x^{(t)}-\bm y^{(t)}\|^2
    \leq
    9(n-1)^2\eta^2G^2
    \quad\text{almost surely}.
\]
At restart times, $\bm x^{(S_r)}=\bm y^{(S_r)}$, so the gap is zero.
\end{lemma}

\begin{proof}
For $t>r(t)$, denote $\bm a_t=\bm x^{(t)}-\bm y^{(t)}$.  Exactly as in the deterministic
proof, unrolling the two recursions after $r(t)$ gives
\begin{align*}
    \bm a_t
    &=
    \frac{\eta}{n}\sum_{i=1}^n\bm g_i^{(r(t))}
    -\sum_{k=r(t)+1}^{t}
        \eta_k\bm g_{j_k}^{(k-\tau_k)}
    +\sum_{k=r(t)+1}^{t-1}
        \hat\eta_k\bm g_{j_k}^{(k)}\\
    &=
    \sum_{m\in[n]\setminus\{j_t\}}
    \eta_{\nexxt(t,m)}
    \bm g_m^{(\prev(t,m))}.
\end{align*}
Because $k-\tau_k\geq r(k)$, all but at most one pending vector per client
cancel.  Canceled computations have weight zero, so at most $n-1$ terms remain.
First-arrival contributions have norm at most $3G$ by
Assumption~\ref{ass:stoch-bounded-local-contributions}, and repeated-arrival
contributions have norm at most $G\leq3G$ by clipping at radius $G$.
Using $\eta_{\nexxt(t,m)}\leq\eta$ therefore gives
\begin{align*}
    \|\bm x^{(t)}-\bm y^{(t)}\|^2
    &\leq
    (n-1)
    \sum_{m\in[n]\setminus\{j_t\}}
    \eta_{\nexxt(t,m)}^2
    \|\bm g_m^{(\prev(t,m))}\|^2\\
    &\leq
    9(n-1)^2\eta^2G^2.
\end{align*}
At restart times the gap is zero by construction.  This bound is pathwise
and therefore also holds after taking expectation.
\end{proof}

\begin{lemma}[Descent Lemma for Non-Restart Rounds]
\label{lem:stoch-hard-non-restart-descent}
If $t>r(t)$ and $\hat\eta_t\leq1/(2L)$, then
\begin{align*}
    \E[f(\bm z^{(t+1)})\mid\gF_t]
    &\leq
    f(\bm z^{(t)})
    -\frac{\hat\eta_t}{4}
    \|\nabla f(\bm x^{(t)})\|^2
    -\frac{\hat\eta_t}{2}
    \|\nabla f(\bm z^{(t)})\|^2\\
    &\quad+
    L^2\hat\eta_t
    \|\bm x^{(t)}-\bm y^{(t)}\|^2
    +L^2\hat\eta_t
    \|\bm y^{(t)}-\bm z^{(t)}\|^2\\
    &\quad+
    \frac{L\sigma^2\hat\eta_t^2}{2}.
\end{align*}
\end{lemma}

\begin{proof}
By $L$-smoothness and conditional unbiasedness,
\begin{align*}
    \E[f(\bm z^{(t+1)})\mid\gF_t]
    &\leq
    f(\bm z^{(t)})
    -\hat\eta_t
    \langle
        \nabla f(\bm z^{(t)}),
        \nabla f(\bm x^{(t)})
    \rangle\\
    &\quad+
    \frac{L\hat\eta_t^2}{2}
    \left(
        \|\nabla f(\bm x^{(t)})\|^2+\sigma^2
    \right).
\end{align*}
Using
\[
    -\langle\bm a,\bm b\rangle
    =
    -\frac12\|\bm a\|^2
    -\frac12\|\bm b\|^2
    +\frac12\|\bm a-\bm b\|^2,
\]
we obtain
\begin{align*}
    \E[f(\bm z^{(t+1)})\mid\gF_t]
    &\leq
    f(\bm z^{(t)})
    -\left(
        \frac{\hat\eta_t}{2}
        -\frac{L\hat\eta_t^2}{2}
    \right)
    \|\nabla f(\bm x^{(t)})\|^2\\
    &\quad-
    \frac{\hat\eta_t}{2}
    \|\nabla f(\bm z^{(t)})\|^2
    +\frac{\hat\eta_t}{2}
    \|\nabla f(\bm z^{(t)})-
      \nabla f(\bm x^{(t)})\|^2\\
    &\quad+
    \frac{L\sigma^2\hat\eta_t^2}{2}.
\end{align*}
Finally,
\begin{align*}
    \frac{\hat\eta_t}{2}
    \|\nabla f(\bm z^{(t)})-
      \nabla f(\bm x^{(t)})\|^2
    &\leq
    L^2\hat\eta_t
    \|\bm x^{(t)}-\bm y^{(t)}\|^2\\
    &\quad+
    L^2\hat\eta_t
    \|\bm y^{(t)}-\bm z^{(t)}\|^2,
\end{align*}
and $\hat\eta_t\leq1/(2L)$ implies
$\hat\eta_t/2-L\hat\eta_t^2/2\geq\hat\eta_t/4$.
\end{proof}

\begin{lemma}[Descent Lemma for Restart Rounds]
\label{lem:stoch-hard-restart-descent}
If $t=r(t)$ and $4L\eta\leq1$, then $\hat\eta_t=\eta$ and
\begin{align*}
    \E[f(\bm z^{(t+1)})\mid\gF_t]
    &\leq
    f(\bm z^{(t)})
    -\frac{\eta}{4}
    \|\nabla f(\bm x^{(t)})\|^2\\
    &\quad+
    L^2\eta
    \|\bm x^{(t)}-\bm y^{(t)}\|^2
    +L^2\eta
    \|\bm y^{(t)}-\bm z^{(t)}\|^2\\
    &\quad+
    \frac{1}{\eta}
    \|\bm y^{(t)}-\bm z^{(t)}\|^2
    +\frac{L\eta^2\sigma^2}{2|\gG|}.
\end{align*}
\end{lemma}

\begin{proof}
Conditional independence gives
\begin{align*}
    \E\left[
        \frac{1}{|\gG|}\sum_{i\in\gG}
        \nabla F(\bm x^{(t)};\xi_i^{(t)})
        \mathrel{\Big|}\gF_t
    \right]
    &=\nabla f(\bm x^{(t)}),\\
    \E\left[
        \left\|
            \frac{1}{|\gG|}\sum_{i\in\gG}
            \nabla F(\bm x^{(t)};\xi_i^{(t)})
            -\nabla f(\bm x^{(t)})
        \right\|^2
        \mid\gF_t
    \right]
    &\leq\frac{\sigma^2}{|\gG|}.
\end{align*}
The restart update is
\[
    \bm z^{(t+1)}
    =
    \bm z^{(t)}
    +\bm y^{(t)}-\bm z^{(t)}
    -\frac{\eta}{|\gG|}\sum_{i\in\gG}
    \nabla F(\bm x^{(t)};\xi_i^{(t)}).
\]
By $L$-smoothness and conditional unbiasedness,
\begin{align*}
    \E[f(\bm z^{(t+1)})\mid\gF_t]
    &\leq
    f(\bm z^{(t)})
    -\eta\langle
        \nabla f(\bm z^{(t)}),
        \nabla f(\bm x^{(t)})
    \rangle\\
    &\quad+
    \left\langle
        \nabla f(\bm z^{(t)}),
        \bm y^{(t)}-\bm z^{(t)}
    \right\rangle
    +L\eta^2\|\nabla f(\bm x^{(t)})\|^2\\
    &\quad+
    L\|\bm y^{(t)}-\bm z^{(t)}\|^2
    +\frac{L\eta^2\sigma^2}{2|\gG|}.
\end{align*}
The same inner-product identity as above and Young's inequality give
\begin{align*}
    \left\langle
        \nabla f(\bm z^{(t)}),
        \bm y^{(t)}-\bm z^{(t)}
    \right\rangle
    \leq
    \frac{\eta}{2}\|\nabla f(\bm z^{(t)})\|^2
    +\frac{1}{2\eta}\|\bm y^{(t)}-\bm z^{(t)}\|^2.
\end{align*}
Moreover,
\begin{align*}
    \frac{\eta}{2}
    \|\nabla f(\bm z^{(t)})-
      \nabla f(\bm x^{(t)})\|^2
    &\leq
    L^2\eta
    \|\bm x^{(t)}-\bm y^{(t)}\|^2\\
    &\quad+
    L^2\eta
    \|\bm y^{(t)}-\bm z^{(t)}\|^2.
\end{align*}
Combining these estimates gives the coefficient
\[
    -\eta\left(\frac12-L\eta\right)
\]
in front of $\|\nabla f(\bm x^{(t)})\|^2$.  The total coefficient in front
of $\|\bm y^{(t)}-\bm z^{(t)}\|^2$ is
\[
    L^2\eta+L+\frac{1}{2\eta}.
\]
Since $4L\eta\leq1$,
\[
    \frac12-L\eta\geq\frac14,
    \qquad
    L+\frac{1}{2\eta}\leq\frac{3}{4\eta}\leq\frac{1}{\eta},
\]
which proves the result.
\end{proof}

\begin{theorem}[Combined Stochastic Descent Lemma]
\label{thm:stoch-hard-combined-descent}
If $\eta\leq1/(4L)$, then for all $t$,
\begin{align*}
    \E[f(\bm z^{(t+1)})\mid\gF_t]
    &\leq
    f(\bm z^{(t)})
    -\frac{\hat\eta_t}{4}
    \|\nabla f(\bm x^{(t)})\|^2\\
    &\quad+
    L^2\hat\eta_t
    \|\bm x^{(t)}-\bm y^{(t)}\|^2
    +L^2\hat\eta_t
    \|\bm y^{(t)}-\bm z^{(t)}\|^2\\
    &\quad+
    \frac{L\hat\eta_t^2\sigma^2}{2}
    \left(
        \indi_{\{t>r(t)\}}
        +\frac{\indi_{\{t=r(t)\}}}{|\gG|}
    \right)
    +\frac{1}{\eta}
    \|\bm y^{(t)}-\bm z^{(t)}\|^2
    \indi_{\{t=r(t)\}}.
\end{align*}
\end{theorem}

\begin{proof}
If $t>r(t)$, apply
Lemma~\ref{lem:stoch-hard-non-restart-descent} and drop the nonpositive term
\[
    -\frac{\hat\eta_t}{2}
    \|\nabla f(\bm z^{(t)})\|^2.
\]
If $t=r(t)$, apply
Lemma~\ref{lem:stoch-hard-restart-descent}.  These two cases prove the stated
inequality.
\end{proof}

\begin{lemma}[Summed restart perturbation bound]
\label{lem:stoch-hard-summed-restart-perturbation}
Let
\[
    N_T
    \coloneqq
    1+\sum_{t=1}^{T-1}\indi_{\{t=r(t)\}}.
\]
Then
\begin{align*}
    \frac{1}{\eta}\sum_{t=1}^{T-1}
    \E\|\bm y^{(t)}-\bm z^{(t)}\|^2
    \indi_{\{t=r(t)\}}
    &\leq
    \eta
    \left(
        4c\delta\sigma^2
        +\frac{8G^2|\gB|^2}{(q-1)^2}
    \right)(N_T-1)
    \\
    &\leq
    \eta
    \left(
        4c\delta\sigma^2
        +\frac{8G^2|\gB|^2}{(q-1)^2}
    \right)N_T.
\end{align*}
\end{lemma}

\begin{proof}
Fix a restart time $t$.  Applying
Lemma~\ref{lem:stoch-hard-virtual-iterate-relationship} with index $t-1$
gives
\[
    \E\|\bm y^{(t)}-\bm z^{(t)}\|^2
    \leq
    \eta^2
    \left(
        4c\delta\sigma^2
        +\frac{8G^2|\gB|^2}{(q-1)^2}
    \right).
\]
Dividing by $\eta$ and summing over the restart times proves the claim.
\end{proof}

\begin{proof}[Proof of Theorem~\ref{thm:main}]
Apply Lemma~\ref{lem:stoch-hard-initial-descent}, weaken its descent
coefficient from $1/2$ to $1/4$, and then sum
Theorem~\ref{thm:stoch-hard-combined-descent} over $t=1,\ldots,T-1$.
Taking total expectation, telescoping from $f(\bm x^{(0)})$, and using
$f(\bm z^{(T)})\geq f_*$ gives
\begin{align*}
    \frac14
    \left(
        \eta\|\nabla f(\bm x^{(0)})\|^2
        +
        \sum_{t=1}^{T-1}\hat\eta_t
        \E\|\nabla f(\bm x^{(t)})\|^2
    \right)
    &\leq
    \Delta
    +L^2\sum_{t=1}^{T-1}\hat\eta_t
        \E\|\bm x^{(t)}-\bm y^{(t)}\|^2\\
    &\quad+
    L^2\sum_{t=1}^{T-1}\hat\eta_t
        \E\|\bm y^{(t)}-\bm z^{(t)}\|^2\\
    &\quad+
    \frac{L\eta^2\sigma^2}{2|\gG|}
    +
    \frac{L\sigma^2}{2}
    \sum_{\substack{t=1\\t>r(t)}}^{T-1}\hat\eta_t^2\\
    &\quad+
    \frac{L\sigma^2}{2|\gG|}
    \sum_{\substack{t=1\\t=r(t)}}^{T-1}\hat\eta_t^2
    +4c\delta\sigma^2\eta(N_T-1)\\
    &\quad+
    \frac{8G^2|\gB|^2}{(q-1)^2}\eta(N_T-1).
\end{align*}
Lemma~\ref{lem:stoch-hard-real-virtual-difference},
Lemma~\ref{lem:stoch-hard-virtual-iterate-relationship}, and the definition
of $T_{\mathrm{eff}}$ imply
\[
    \sum_{\substack{t=1\\t>r(t)}}^{T-1}\hat\eta_t
    =
    \eta(T_{\mathrm{eff}}-N_T),
\]
and, since the real--virtual gap is zero at restart times,
\begin{align*}
    L^2\sum_{t=1}^{T-1}\hat\eta_t
        \E\|\bm x^{(t)}-\bm y^{(t)}\|^2
    &=
    L^2\sum_{\substack{t=1\\t>r(t)}}^{T-1}\hat\eta_t
        \E\|\bm x^{(t)}-\bm y^{(t)}\|^2\\
    &\leq
    9L^2(n-1)^2G^2\eta^3(T_{\mathrm{eff}}-N_T)\\
    &=
    9L^2(n-1)^2G^2\eta^3T_{\mathrm{eff}}(1-\theta_T),\\
    L^2\sum_{\substack{t=1\\t>r(t)}}^{T-1}\hat\eta_t
        \E\|\bm y^{(t)}-\bm z^{(t)}\|^2
    &\leq
    L^2\eta^3(T_{\mathrm{eff}}-N_T)
    \left(
        4c\delta\sigma^2
        +\frac{8G^2|\gB|^2}{(q-1)^2}
    \right).
\end{align*}
At restart times, the remaining part of the same sum satisfies
\begin{align*}
    &L^2\sum_{t=1}^{T-1}\hat\eta_t
        \E\|\bm y^{(t)}-\bm z^{(t)}\|^2
        \indi_{\{t=r(t)\}}\\
    &\quad\leq
    L^2\eta^3
    \left(
        4c\delta\sigma^2
        +\frac{8G^2|\gB|^2}{(q-1)^2}
    \right)(N_T-1).
\end{align*}
Since $L^2\eta^2\leq1/16$, combining this restart contribution with
Lemma~\ref{lem:stoch-hard-summed-restart-perturbation} gives at most
\[
    \frac{17}{16}\eta
    \left(
        4c\delta\sigma^2
        +\frac{8G^2|\gB|^2}{(q-1)^2}
    \right)(N_T-1)
    \leq
    \frac{17}{16}\eta
    \left(
        4c\delta\sigma^2
        +\frac{8G^2|\gB|^2}{(q-1)^2}
    \right)N_T.
\]
Also, $\hat\eta_t=\eta$ at restart times, while
$\hat\eta_t^2\leq\eta\hat\eta_t$ otherwise.  Including the initial step,
the variance weights satisfy
\begin{align*}
    &\frac{\eta^2}{|\gG|}
    +\sum_{\substack{t=1\\t>r(t)}}^{T-1}\hat\eta_t^2
    +\frac{1}{|\gG|}
        \sum_{\substack{t=1\\t=r(t)}}^{T-1}\hat\eta_t^2\\
    &\quad=
    \frac{\eta^2N_T}{|\gG|}
    +\sum_{\substack{t=1\\t>r(t)}}^{T-1}\hat\eta_t^2\\
    &\quad\leq
    \eta^2\left(\frac{N_T}{|\gG|}+T_{\mathrm{eff}}-N_T\right)\\
    &\quad=
    \eta^2T_{\mathrm{eff}}
    \left(1-\theta_T+\frac{\theta_T}{|\gG|}\right).
\end{align*}
Collecting these estimates yields
\begin{align*}
    \frac14
    \left(
        \eta\|\nabla f(\bm x^{(0)})\|^2
        +
        \sum_{t=1}^{T-1}\hat\eta_t
        \E\|\nabla f(\bm x^{(t)})\|^2
    \right)
    &\leq
    \Delta
    +\frac{L\sigma^2\eta^2T_{\mathrm{eff}}}{2}
        \left(1-\theta_T+\frac{\theta_T}{|\gG|}\right)\\
    &\quad+
    9L^2\eta^3T_{\mathrm{eff}}(1-\theta_T)(n-1)^2G^2\\
    &\quad+
    4L^2\eta^3T_{\mathrm{eff}}(1-\theta_T)c\delta\sigma^2\\
    &\quad+
    8L^2\eta^3T_{\mathrm{eff}}(1-\theta_T)
    \frac{G^2|\gB|^2}{(q-1)^2}\\
    &\quad+
    \frac{17}{4}c\delta\sigma^2\eta N_T\\
    &\quad+
    \frac{17}{2}\eta\frac{G^2|\gB|^2}{(q-1)^2}N_T.
\end{align*}
Dividing by $\eta T_{\mathrm{eff}}/4$ and using the output distribution
gives
\begin{align*}
    \E\|\nabla f(\bm x^{\mathrm{out}})\|^2
    &\leq
    \frac{4\Delta}{\eta T_{\mathrm{eff}}}
    +2L\sigma^2\eta
        \left(1-\theta_T+\frac{\theta_T}{|\gG|}\right)\\
    &\quad+
    4L^2\eta^2(1-\theta_T)
    \left[
        9(n-1)^2G^2
        +4c\delta\sigma^2
        +\frac{8G^2|\gB|^2}{(q-1)^2}
    \right]\\
    &\quad+
    \frac{17}{4}
    \left(
        4c\delta\sigma^2
        +\frac{8G^2|\gB|^2}{(q-1)^2}
    \right)
    \theta_T.
\end{align*}
Using the stepsize tuning lemma (Lemma 17~\citep{koloskova2020icml}), we obtain the stated convergence rate.
\end{proof}

\section{Local Robust Aggregators}
\label{app:local_agg}
 
This appendix collects the material on local $(c,\delta)$-robust aggregators (Definition~\ref{def:local_robust_agg}) that is used in Sections~\ref{sec:theory}: 
a general construction with examples (Appendix~\ref{app:local_agg_construction}) and
the impossibility of decomposition without side information (Appendix~\ref{app:local_agg_impossibility}).
Assumption~\ref{ass:stoch-bounded-local-contributions} for centered clipping is verified in Remark~\ref{rem:stoch-centered-clipping-bounded}.
 
\subsection{A General Construction}
\label{app:local_agg_construction}
 
Every local aggregator we use is of the following form.
 
\begin{lemma}[Centered aggregators are decomposable]
\label{lem:anchored_local}
Let $\psi : \mathbb{R}^d \to \mathbb{R}^d$ be any map and take the side information to be an anchor, $S = \bm a \in \mathbb{R}^d$.
Then
\[
    \mathbb{A}(\bm v; \bm a) := \bm a + \psi(\bm v - \bm a),
    \qquad
    \Agg(\bm x_1, \dots, \bm x_n; \bm a) := \bm a + \frac{1}{n} \sum_{i=1}^n \psi(\bm x_i - \bm a),
\]
satisfy $\Agg(\bm x_1, \dots, \bm x_n; \bm a) = \frac{1}{n} \sum_{i=1}^n  \mathbb{A}(\bm x_i; \bm a)$.
\end{lemma}
\begin{proof}
The identity is immediate from $\frac1n \sum_i (\bm a + \psi(\bm x_i - \bm a)) = \bm a + \frac1n \sum_i \psi(\bm x_i - \bm a)$.
\end{proof}
 
The role of $\psi$ is to leave honest inputs (essentially) unchanged, $\psi(\bm u) \approx \bm u$ when $\|\bm u\|$ is of the order of the honest inputs' magnitude, while bounding the effect of arbitrary inputs.
Two natural choices are
\begin{align*}
    \text{Centered Clipping~\citep{karimireddy2021icml}:}\quad & \psi_{\mathrm{cc}}(\bm u) := \bm u \min\Big\{1, \frac{\tau}{\|\bm u\|}\Big\}, \\
    \text{Centered Truncated Mean (CTM, new)}\quad & \psi_{\mathrm{ctm}}(\bm u) := \bm u \, \mathbbm{1}_{\{\|\bm u\| \le R\}}.
\end{align*}
Centered clipping with $S = (\bm a, \tau)$ is Example~\ref{ex:centered_clipping}. Centered clipping is $(c,\delta)$-robust when the anchor is close to the honest mean and the radius is chosen accordingly \citep[Theorem~IV]{karimireddy2021icml}.
Here, we show that CTM is also a local $(c,\delta)$-robust aggregator.
\begin{proposition}[Centered Truncated Mean is a $(c,\delta)$-robust aggregator]
\label{prop:ahtm}
Let $\mathrm{CTM}_{a,R}$ be the aggregator in the preceding theorem with $\psi_{\mathrm{ctm}}$.
Let $n\ge 2$ and $\frac{1}{n}\leq \delta < \frac12.$
Let $\gG\subseteq[n]$ denote the set of honest inputs and $\gB=[n]\setminus\gG$ the set of Byzantine inputs, with $\beta := \frac{|\gB|}{n}\leq \delta.$
Fix the side information $(\bm a, R)$ before observing the current inputs. Suppose that, conditional on this side information, the honest inputs are independent and satisfy
$$\E \| \bm x_i - \bm a \|^2 \leq V^2, \quad \forall i \in \gG,$$
for some $V \geq 0.$
Choosing $R^2 = \frac{V^2}{\delta}$ yields that
$$\E \| \mathrm{CTM}_{a,R}(\bm x_1, \bm x_2 , \ldots, \bm x_n) - \bar x\|^2 \leq 10 \delta V^2.$$
Consequently, if the admissible side information pair satisfies $V^2 \leq K \rho^2$ for some constant $K$, then Centered Truncated Mean satisfies the definition of $(c,\delta)$-robust aggregator with $c = 10K.$
\end{proposition}

\begin{proof}
Write
$$ \psi_R(\bm u)=\bm u\mathbbm{1}_{\{\|\bm u\|\leq R\}}, \quad Z_i=(\bm x_i-\bm a)\mathbbm{1}_{\{\|\bm x_i- \bm a\|>R\}}, \quad i\in \gG. $$
By assumption, $\E \|Z_i\|^2 \leq V^2$.
Also, by Markov's inequality,
\begin{align*}
    \|\E Z_i\| &\leq \E \left[ \|\bm x_i- \bm a\|\mathbbm{1}_{\{\|\bm x_i- \bm a\|>R\}} \right]\\ 
    &\leq \frac{\mathbb E\|x_i-a\|^2}{R} \le \frac{V^2}{R}.
\end{align*}

Next, the aggregation error admits the decomposition
$$ \mathrm{CTM}_{a,R} (\bm x_1, \ldots, \bm x_n) -\bar x = -\frac1n\sum_{i\in\gG}Z_i +  \frac1n\sum_{i\in\gB}\psi_R(\bm x_i-\bm a) -\beta(\bar{\bm x} - \bm a) . $$
Because \(\|\psi_R(u)\|\leq R\),
$$ \left\| \frac1n\sum_{i\in\gB}\psi_R(\bm x_i-\bm a) \right\|^2 \leq \beta R.$$
Furthermore, Jensen's inequality gives
$$ \mathbb E\|\bar{\bm x}- \bm a\|^2 \le \frac1{|\mathcal G|}\sum_{i\in\gG}\E\|\bm x_i- \bm a\|^2 \leq V^2. $$
Therefore,
$$ \mathbb E\left\| \frac1n\sum_{i\in\gB}\psi_R(\bm x_i-\bm a) -\beta(\bar{\bm x} - \bm a) \right\|^2 \le 2\beta^2(R^2+V^2). $$
Combining above result yields
\begin{align*}    
\mathbb E\| \mathrm{CTM}_{a,R} (\bm x_1, \ldots, \bm x_n)-\bar{\bm x} \|^2 
&\leq \frac{2V^2}{n} +\frac{2V^4}{R^2} +4\delta^2(R^2+V^2) \\
&\leq \left(\frac2n+6\delta+4\delta^2\right)V^2 \\
&\leq 10\delta V^2
\end{align*}
by $R^2=V^2/\delta$ and $\frac 1n\leq \delta <1/2 $.
The case $V=0$ is immediate with $R=0$.
\end{proof}
 
\subsection{Impossibility Without Side Information}
\label{app:local_agg_impossibility}
 
Definition~\ref{def:local_robust_agg} allows the local map $\mathbb{A}$ to depend on a side information $S$.
The following proposition shows that this is not a technical convenience: no per-input map of the input alone can be robust.
 
\begin{proposition}[Failure without side information]
\label{prop:no_side_info}
Let $n \ge 3$, $d \ge 1$, and $1/n \le \delta < 1/2$.
For any finite $c \ge 0$, there is no map
$f : \R^d \to \R^d$ such that the aggregator
\[
    \Agg(\bm x_1,\ldots,\bm x_n)
    := \frac{1}{n}\sum_{i=1}^n f(\bm x_i)
\]
is $(c,\delta)$-robust in the sense of
Definition~\ref{def:robust_agg}.
\end{proposition}

\begin{proof}
Suppose such a map $f$ exists.
Let $\gG=[n-1]$ and $\gB=\{n\}$.
This split is admissible because
$|\gG|=n-1\ge(1-\delta)n$.

Fix arbitrary $\bm a,\bm y\in\R^d$, and take the deterministic
inputs
\[
    \bm x_1=\cdots=\bm x_{n-1}=\bm a,
    \qquad \bm x_n=\bm y.
\]
The honest inputs are independent, their mean is $\bm a$,
and their pairwise dispersion is zero.
Definition~\ref{def:robust_agg}, applied with $\rho=0$,
therefore gives
\[
    \frac{n-1}{n}f(\bm a)+\frac{1}{n}f(\bm y)
    = \bm a
    \qquad\text{for all }\bm a,\bm y\in\R^d.
\]
Taking $\bm y=\bm a$ shows that $f(\bm a)=\bm a$
for every $\bm a\in\R^d$.
Substituting this identity back into the display yields
\[
    \frac{n-1}{n}\bm a+\frac{1}{n}\bm y=\bm a
    \qquad\text{for all }\bm a,\bm y\in\R^d,
\]
which is impossible when $\bm y\ne\bm a$.
\end{proof}

\raggedbottom
\section{Additional Experimental Details}~\label{app:experiments}

\begin{figure}[t]
    \centering
    \includegraphics[width=1\linewidth]{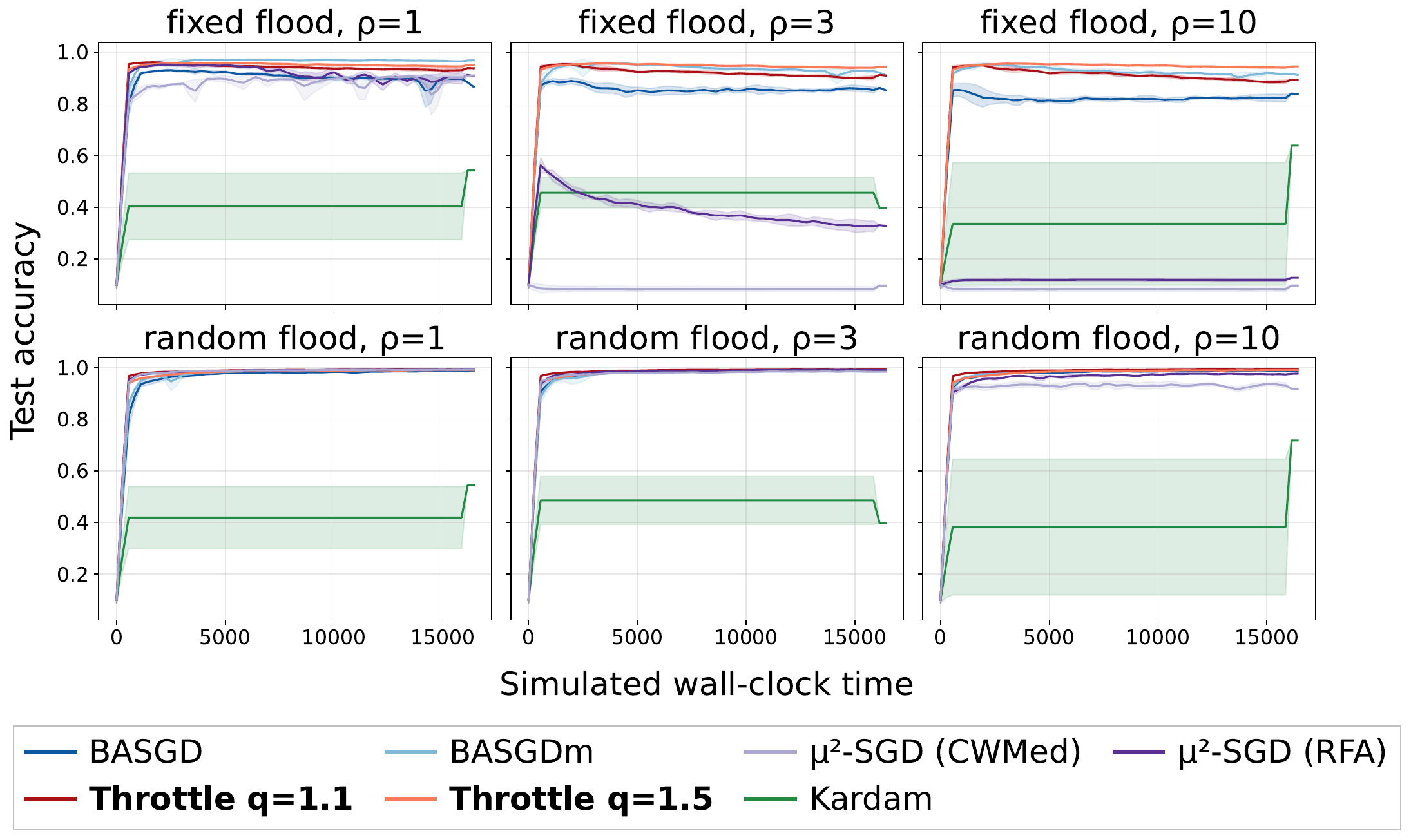}
    \caption{Additional MNIST flooding-attack results complementing
Figure~\ref{fig:byzantine-exp}, with $n=25$ clients, including
$|\gB|=5$ Byzantine clients. The top and bottom rows show
fixed-vector and random-vector flooding, respectively. Columns
correspond to $\rho\in\{1,3,10\}$, extending the $\rho=30$ results
in Figure~\ref{fig:byzantine-exp}. Here, $\rho$ multiplies the
Byzantine clients' update rates, yielding an expected Byzantine
update fraction of $\rho/(2+\rho)$. All other experimental
settings are the same as in Figure~\ref{fig:byzantine-exp}.
\textsc{Throttle} maintains high test accuracy across the
three flooding rates.}
    \label{fig:appendix-exp-figure}
\end{figure}

\subsection{Common settings and implementation}
\label{app:common}

Two codebases were used. 
The least-squares experiment of Figure~\ref{fig:no-byzantine-async-throttling} runs
\emph{real} asynchronous workers. 
Every client is a Ray actor, the parameter server applies gradients in the order in which they arrive, and delays are produced by CPU contention rather than by a model. 
The MNIST experiment of Figure~\ref{fig:byzantine-exp} uses a single-process, event-driven simulator with simulated
(virtual) time based on the \citep{dahan_weight_2024} codebase.
The arrival sequence of every client is generated once per (schedule, seed) and is identical for all methods, communication latency is zero, and a client's next computation starts when its previous message arrives. 
Kardam and $\mu^2$-SGD were ported from the authors' public github repositories (\texttt{gdamaskinos/kardam} and
\texttt{dahan198/asynchronous-fault-tolerant-ml}).
BASGD, BASGDm were implemented from the papers, since no public code exists. 
Each seed spawns independent random streams for the data order, the arrival schedule, the initialization and the attack. Tables~\ref{tab:quadratic-config} and~\ref{tab:mnist-config} list the complete configurations.

\begin{table}[t]
\centering\small
\caption{Least-squares configuration used in Figure~\ref{fig:no-byzantine-async-throttling}.}
\label{tab:quadratic-config}
\begin{tabular}{ll}
\toprule
Objective & $F(x)=\frac{1}{2n}\|Ax-b\|^2$, $n=10{,}000$, $d=400$, condition number $\approx1.9\times10^{3}$ \\
Stochastic gradient & minibatch of 256 samples without replacement. all clients have shared i.i.d.\ data \\
Clients & $M=40$ Ray actors on 20 CPU cores.  \\
Straggler (right panel) & worker 39 sleeps $100\times$ its compute time after each gradient ($1.7\%$ of arrivals instead of $2.5\%$) \\
Budget & $32{,}000$ computed gradients for every method. \\
Methods & Minibatch SGD and asynchronous SGD are based on \citep{mishchenko_asynchronous_2023}, \\
        & \textsc{Throttle} $q\in\{1.1,1.5\}$, $\mu^2$-SGD ($\beta=0.25,\gamma=0.1$), \\
        & BASGD / BASGDm ($B=10$, mean, $\mu=0.9$) \\
\textsc{Throttle} & no clipping ($\lambda=\tau=\infty$, no Byzantine clients) \\
                 & hard restart when all $M$ clients have arrived; discarded gradients included in the budget \\
Step-size tuning & grid search over $\eta=2^{k/3}, \, k \in [-30, 18] \cap \mathbb{Z}$ (Table~\ref{tab:quadratic-lr}) \\
Seeds & $\{42,43,44\}$ \\
\bottomrule
\end{tabular}
\centering
\small\caption{Learning rates used in Fig.~\ref{fig:no-byzantine-async-throttling}. Grid $\eta=2^{k/3}$ (displayed learning rates are rounded); criterion: mean of $F(x)-F^*$ over the checkpoints in the last 10\% of the 32{,}000-computed gradient budget, averaged over seeds $\{42,43,44\}$; a grid point is valid only if all three repeats converge. `tail mean' is that criterion at the selected $\eta$.}
\label{tab:quadratic-lr}
\begin{tabular}{llc}
\toprule
Method & $\eta$ used & tail mean $\pm$ std \\
\midrule
Minibatch SGD ($q=\infty$) & $0.1984$, $k=-7$ & $3.46\times10^{-3} \pm 3.4\times10^{-5}$ \\
Asynchronous SGD ($q=1$) & $0.5$, $k=-3$ & $1.37\times10^{-4} \pm 2.0\times10^{-5}$ \\
Throttle ($q=1.1$) & $1.587$, $k=2$ & $7.14\times10^{-6} \pm 7.4\times10^{-6}$ \\
Throttle ($q=1.5$) & $3.175$, $k=5$ & $2.55\times10^{-6} \pm 1.1\times10^{-6}$ \\
Asynchronous $\mu^2$-SGD & $0.125$, $k=-9$ & $1.82\times10^{-2} \pm 1.4\times10^{-5}$ \\
BASGD ($B=10$) & $5.04$, $k=7$ & $8.89\times10^{-3} \pm 2.4\times10^{-3}$ \\
BASGDm ($B=10$, $\mu=0.9$) & $5.04$, $k=7$ & $6.99\times10^{-3} \pm 3.0\times10^{-4}$ \\
\bottomrule
\end{tabular}
\end{table}

\subsection{Least-squares experiment (Figure~\ref{fig:no-byzantine-async-throttling})}
\label{app:quadratic}

\begin{figure}[!htbp]
    \centering
    \includegraphics[width=0.65\linewidth]{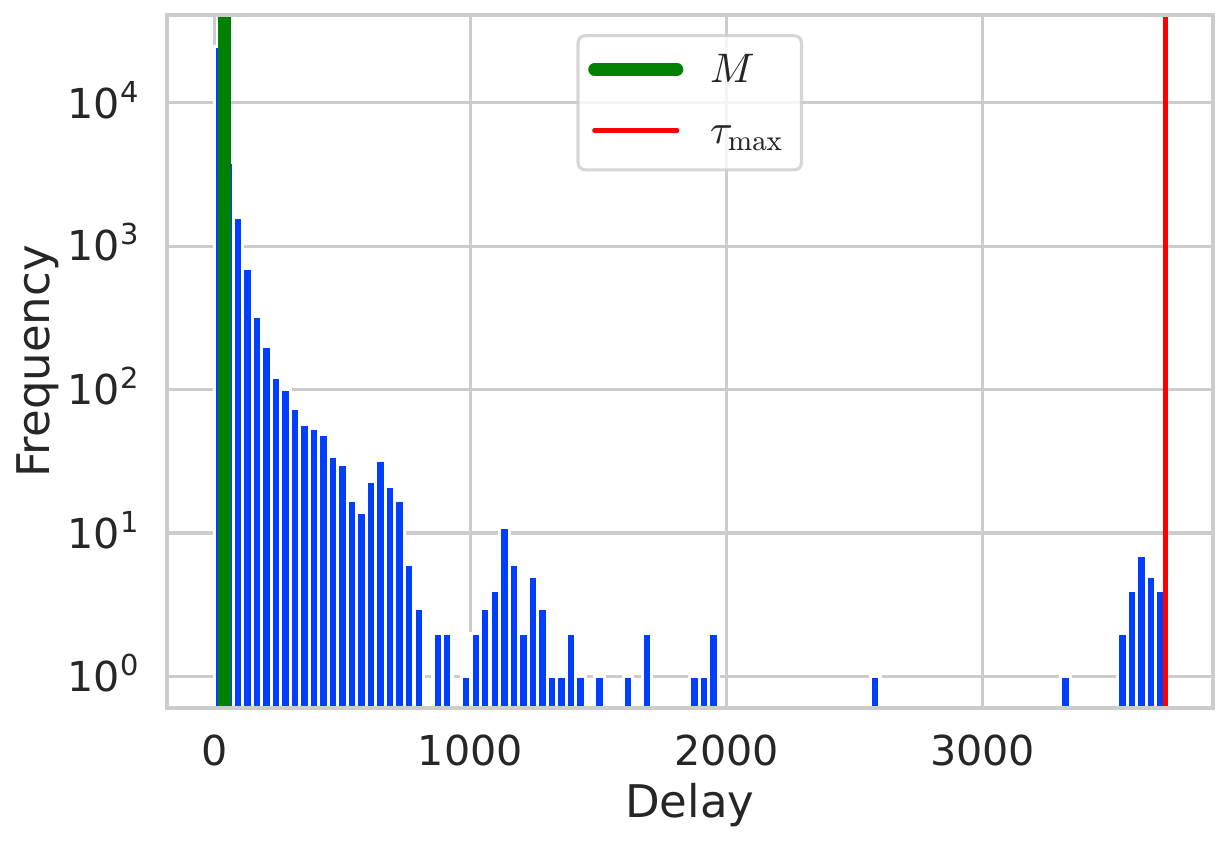}
    \caption{Empirical distribution of gradient delays in the least-squares experiment with $M=40$ asynchronous workers on 20 CPU cores based on \citep{mishchenko_asynchronous_2023}. Delay is the number of server updates between a worker's dispatch and the arrival of its gradient. 
    frequencies are shown on a logarithmic scale.
    The green and red vertical lines average delay $M$ and the maximum observed delay $\tau_{\max}$, respectively, illustrating that delays can greatly exceed the number of workers.}
    \label{fig:delay_plot}
\end{figure}

\paragraph{Problem.}
We minimize $F(x)=\frac{1}{2n}\|Ax-b\|^2$ with $n=10{,}000$ samples and $d=400$. The entries of $A$ are
i.i.d.\ uniform on $[0,1)$ scaled by $1/\sqrt{d}$. 
A stochastic gradient is the gradient of $F$ on a minibatch of 256 samples drawn without replacement; all workers hold the same data (i.i.d., no
partitioning). 
Reported losses are $F(x)-F^*$ evaluated on all $n$ samples every 40 accepted gradients.

\paragraph{Asynchronous execution.}
$M=40$ workers run as Ray actors on a 20-core CPU machine (two workers per core), so arrival times are genuine wall-clock timings. 
The server applies each gradient on arrival and immediately re-dispatches the same worker with the updated model, i.e.\ there are always 40 computing gradients. 
Deley distribution (the number of server updates between a worker's dispatch and its arrival) is naturally heavy-tailed.
See Fig.~\ref{fig:delay_plot} for our experiment environments' case, similarly computed as in \citet{mishchenko_asynchronous_2023}.

In the right panel of Figure~\ref{fig:no-byzantine-async-throttling} one worker sleeps for 100 times its own compute time after every gradient. Figure~\ref{fig:no-byzantine-async-throttling} shows the seed 42 run of every method.

\subsection{MNIST experiment (Figure~\ref{fig:byzantine-exp})}
\label{app:mnist}

\begin{table}[H]
\centering\small
\caption{MNIST configuration used in Figure~\ref{fig:byzantine-exp}.}
\label{tab:mnist-config}
\begin{tabular}{ll}
\toprule
Dataset & MNIST (60,000 train / 10,000 test), i.i.d. clients\\
& each drawing minibatches from its own reshuffling of the full training set \\
Model & Conv-Conv-FC-BatchNorm-FC (66{,}230 parameters), cross-entropy loss \\
Clients & $n=20$ (15 honest, 5 Byzantine. $\delta=0.25$) \\
Arrivals (standard attacks) & periodic sequence as in the \citet{dahan_weight_2024}'s code \\
& every third arrival is Byzantine ($1/3$ of updates). \\
          & The arriving client is drawn within its group with probability $\propto$ its index \\
Arrivals (flooding) & Independent Poisson processes: honest client $i$ has rate $r_i\propto i$ (honest rates sum to 1); \\
          & the Byzantine group has total rate $\rho/2$, split $\propto$ index, $\rho\in\{1,3,10,30\}$, \\
          & i.e.\ an expected Byzantine share of updates $\rho/(2+\rho)=0.33,\,0.60,\,0.83,\,0.94$ \\
Batch size & 16 per client \\
Training horizon & $16{,}000$ honest updates ($800$ sync.\ iterations, $4.3$ epochs); test every 500  \\
Attacks & RD ($\sigma=0.2$), NG ($-10\times$), Empire ($\epsilon=6$), ALIE ($z_{\max}=0.253$), \\
        & fixed-vector flood, random-vector flood (norm 10) \\
Methods & Kardam ($f=5$); BASGD / BASGDm ($B=11$, trimmed mean, $\Delta=20$, $\mu=0.9$); \\
        & $\mu^2$-SGD (weighted CWMed / RFA; $\beta=0.25$, $\gamma=0.1$); \\
        & \textsc{Throttle} ($q\in\{1.1,1.5\}$, $\tau=100$, $\lambda=1$) \\
$\eta$ tuning & grid $\{1,0.1,0.01\}$ without attack (Table~\ref{tab:mnist-lr-sweep}), clipping radius $\lambda=1$ \\
Seeds & $\{1,2,3\}$ \\
\bottomrule
\end{tabular}
\end{table}

\begin{table}[H]
\centering\small
\caption{Learning-rate selection on MNIST without attack: final test accuracy (\%, mean $\pm$ std over seeds $\{1,2,3\}$) for $\eta\in\{0.01,0.1,1\}$. Throttle results are shown for $\lambda=1$. Bold entries indicate the selected values for the baseline methods; ties are broken towards the smaller $\eta$.}
\label{tab:mnist-lr-sweep}
\begin{tabular}{lccc}
\toprule
Method & $\eta=0.01$ & $\eta=0.1$ & $\eta=1$ \\
\midrule
Kardam & 41.5 $\pm$ 14.7 & \textbf{70.3 $\pm$ 10.5} & 68.5 $\pm$ 26.4 \\
BASGD & 94.4 $\pm$ 0.4 & 98.3 $\pm$ 0.2 & \textbf{98.7 $\pm$ 0.1} \\
BASGDm & 94.7 $\pm$ 0.5 & 98.6 $\pm$ 0.1 & \textbf{99.1 $\pm$ 0.1} \\
$\mu^2$-SGD (CWMed) & 99.1 $\pm$ 0.1 & \textbf{99.2 $\pm$ 0.2} & 10.6 $\pm$ 0.9 \\
$\mu^2$-SGD (RFA) & 99.1 $\pm$ 0.1 & \textbf{99.2 $\pm$ 0.1} & 11.0 $\pm$ 0.6 \\
Throttle $q=1.1$, $\lambda=1$ & 98.8 $\pm$ 0.1 & 99.2 $\pm$ 0.1 & 99.2 $\pm$ 0.1 \\
Throttle $q=1.5$, $\lambda=1$ & 97.9 $\pm$ 0.1 & 99.1 $\pm$ 0.1 & 99.0 $\pm$ 0.1 \\
\bottomrule
\end{tabular}
\end{table}

\subsection{Computing environment}
\label{app:compute}
Both experiments ran on one workstation (Table~\ref{tab:compute}). The least-squares runs use NumPy on the CPU only (40 Ray actors on 20 cores).
The MNIST study (585 runs) took 5.3 hours of wall-clock time, about 74 process-hours.
The least-squares step-size sweeps took about 10 hours and the final 42 runs about 40 minutes, executed strictly one Ray instance at a time.
\begin{table}[H]
    \centering\small
    \caption{Runtime hardware and software.}
    \label{tab:compute}
    \begin{tabular}{ll}
    \toprule
    CPU & 20 logical cores (aarch64) \\
    GPU & NVIDIA GB10, CUDA 13.0 (MNIST, $\mu^2$-SGD only) \\
    Python / NumPy & 3.12.3 / 2.4.4 \\
    PyTorch & 2.12.0+cu130 (MNIST) \\
    Ray & 2.58.0 (least squares) \\
    \bottomrule
    \end{tabular}
\end{table}

\paragraph{Data and model.}
We use the standard MNIST dataset (60{,}000 training and 10{,}000 test images). 
The model is the convolutional network of the \citet{dahan_weight_2024}'s codebase: \texttt{Conv}$(1{\to}20,5{\times}5)$--ReLU--MaxPool$(2)$--\texttt{Conv}$(20{\to}50,5{\times}5)$--ReLU--MaxPool$(2)$--\texttt{FC}$(800{\to}50)$--BatchNorm--ReLU--\texttt{FC}$(50{\to}10)$ with 66{,}230 parameters and the cross-entropy loss.
Every client draws minibatches of 16 from its own
independent reshuffling of the full training set (i.i.d.\ clients). 
Byzantine clients that run the honest protocol (RD, NG) use separate data, so they never consume the honest data budget. 
Test accuracy on the full test set is evaluated every 500 honest arrivals.

\paragraph{Simulator and arrival schedules.}
\label{app:mnist-schedules}
There are $n=20$ clients, of which the last $5$ are Byzantine ($\delta=0.25$). A run ends after
$16{,}000$ honest updates, i.e.\ $800$ synchronous iterations of 20 clients or $4.3$ epochs. 
Two arrival schedules are used.
\emph{Periodic} (standard-attack panels in Figure~\ref{fig:byzantine-exp}) reproduces the \citet{dahan_weight_2024}'s code: arrivals are numbered
$1,2,\dots$, every third arrival is Byzantine (the code rounds its parameter $\lambda_{\mathrm{byz}}=0.4$ to
the period $\lceil1/0.4\rceil=3$, so the Byzantine share of arrivals is $1/3$, not $0.4$).

Within each group (honest and Byzantine), the sending client is drawn with probability proportional to its index plus one.
Therefore, honest client rates differ by a factor of 15 from Byzantine rates. 
For flooding attacks, every client is an independent Poisson process; honest rates are the same proportional weights normalized to sum to one.
The Byzantine group's total rate is $\rho\cdot\frac{p_0}{1-p_0}=\rho/2$ with $p_0=1/3$.
The expected Byzantine share of all arrivals is $\rho p_0/(1-p_0+\rho p_0)=0.33,0.60,0.83,0.94$ for $\rho=1,3,10,30$ (measured: $0.337,0.604,0.835,0.938$)
The total number of arrivals grows from
$24{,}000$ to $260{,}000$ while the honest budget stays fixed.

\paragraph{Attacks.}
Let $g$ be the update an honest client would send and $L$ the set of honest clients whose latest update has been delivered. 
\begin{itemize}
    \item RD (random disturbance, \citealp{yang2023buffered}) sends $g+\xi$, $\xi\sim\mathcal{N}(0,\|0.2\,g\|^2 I)$.
    \item NG (negative gradient, \citealp{yang2023buffered}) sends $-10\,g$.
    \item Empire~\citep{xie_fall_2019} sends $-6\,\bar g$, where $\bar g$ is the mean of the honest updates' most recently delivered to the server.
    \item ALIE~\citep{baruch2019nips}: following \citet{yang2023buffered}, a Byzantine client sends
$\tilde g$ with $\tilde g_j=\mathrm{mean}_j - z^{\max}\,\mathrm{std}_j$, where $\mathrm{mean}_j$ and
$\mathrm{std}_j$ are the mean and standard deviation of the $j$-th coordinate over the honest updates' most recently delivered to the server, and $z^{\max}=\Phi^{-1}\big(\tfrac{m-\lfloor m/2+1\rfloor}{m-r}\big)$ with $m$ clients, $r$ of them Byzantine, and $\Phi^{-1}$ the standard normal quantile function.
\end{itemize}
Empire and ALIE used omniscient property of Byzantine clients in the sense that they read the honest updates that have actually been
delivered to the server. 
All Byzantine updates are generated at send time from the latest available information. 
The two \emph{flooding} attacks require no gradient computation: the fixed-vector flood sends the same random direction of norm $10$ in every message, the random-vector flood a fresh standard normal direction rescaled to norm $10$, both at $\rho$ times the honest group's rate as described above.

\end{document}